\documentclass{article}

\PassOptionsToPackage{sort, numbers}{natbib}

\usepackage[final]{neurips_2020}

\usepackage[utf8]{inputenc} 
\usepackage[T1]{fontenc}    
\usepackage{hyperref}       
\usepackage{url}            
\usepackage{booktabs}       
\usepackage{amsfonts}       
\usepackage{nicefrac}       
\usepackage{microtype}      
\usepackage{gensymb}
\usepackage{amsmath}
\usepackage{mathtools}
\usepackage{isomath}
\usepackage{float}
\usepackage{wrapfig}
\usepackage{multirow}
\usepackage{xcolor}
\usepackage{etoolbox}
\usepackage{siunitx}
\usepackage{todonotes}
\usepackage{bbm}
\usepackage{bm}

\usepackage{amsthm}

\newtheorem{theorem}{Theorem}[section]
\newtheorem{corollary}{Corollary}[theorem]
\newtheorem{lemma}[theorem]{Lemma}
\newtheorem{proposition}[theorem]{Proposition}

\theoremstyle{remark}

\theoremstyle{definition}

\def\eqref#1{equation~\ref{#1}}

\def\1{\bm{1}}

\DeclareMathAlphabet{\mathsfit}{\encodingdefault}{\sfdefault}{m}{sl}
\SetMathAlphabet{\mathsfit}{bold}{\encodingdefault}{\sfdefault}{bx}{n}

\hypersetup{
    colorlinks=true,
    linkcolor=black,
    urlcolor=cyan,
    citecolor=black,
}

\title{On Second Order Behaviour\\in Augmented Neural ODEs}

\author{%
  Alexander Norcliffe \\
  Department of Physics \\
  University of Cambridge \\
  \texttt{alex.norcliffe98@gmail.com}
   \And
  Cristian Bodnar\thanks{corresponding authors} , Ben Day\footnotemark[1] , Nikola Simidjievski, Pietro Li\`o \\
  Department of Computer Science and Technology \\
  University of Cambridge \\
  \texttt{\{cb2015, bjd39, ns779, pl219\}@cam.ac.uk}
}

\begin{document}

\maketitle

\begin{abstract}
Neural Ordinary Differential Equations (NODEs) are a new class of models that transform data continuously through infinite-depth architectures. The continuous nature of NODEs has made them particularly suitable for learning the dynamics of complex physical systems. While previous work has mostly been focused on first order ODEs, the dynamics of many systems, especially in classical physics, are governed by second order laws. In this work, we consider Second Order Neural ODEs (SONODEs). We show how the adjoint sensitivity method can be extended to SONODEs and prove that the optimisation of a first order coupled ODE is equivalent and computationally more efficient. Furthermore, we extend the theoretical understanding of the broader class of Augmented NODEs (ANODEs) by showing they can also learn higher order dynamics with a minimal number of augmented dimensions, but at the cost of interpretability. This indicates that the advantages of ANODEs go beyond the extra space offered by the augmented dimensions, as originally thought. Finally, we compare SONODEs and ANODEs on synthetic and real dynamical systems and demonstrate that the inductive biases of the former generally result in faster training and better performance.  
\end{abstract}

\section{Introduction}
Residual Networks (ResNets) \cite{he2015deep} have been an essential tool for scaling the capabilities of neural networks to extreme depths. It has been observed that the skip layers that these networks employ can be seen as an Euler discretisation of a continuous transformation \citep{lu2017finite, Haber_2017, ruthotto2019deep}. Neural Ordinary Differential Equations (NODEs) \citep{chen2018neural} are a new class of models that consider the limit of this discretisation step, naturally giving rise to an ODE that can be optimised via black-box ODE solvers. Their continuous depth makes them particularly suitable for learning and modelling the unknown dynamics of complex systems, which often cannot be described analytically. 

Since the introduction of NODEs, many variants have been proposed~\citep{jia2019neural, Tzen2019NeuralSD, dupont2019augmented, Zhang2019ANODEV2AC, yldz2019ode2vae, poli2019graph, massaroli2020dissecting}. While a few of these models use second order dynamics \citep{yldz2019ode2vae, poli2019graph, massaroli2020dissecting}, no in-depth study on second order behaviour in Neural ODEs exists even though most dynamical systems that arise in science, such as Newton's equations of motion and oscillators, are governed by second order laws. To fill this void, we take a deeper look at Second Order Neural ODEs (SONODEs) and the broader class of models formed by Augmented Neural ODEs (ANODEs). Unlike previous approaches, which mainly focus on classification tasks, we use low-dimensional physical systems, often with known analytic solutions, as our main arena of investigation. As we will show, the simplicity of these systems is useful in analysing the properties of these models.

To summarise our contributions, we begin by studying more closely the optimisation of SONODEs by generalising the adjoint sensitivity method to second order models. We continue by analysing how some of the properties of ANODEs extend to SONODEs and show that the latter can often find simpler solutions for the problems we consider. Our analysis also extends to ANODEs and demonstrates that they are capable of learning higher-order dynamics, sometimes with just a few additional dimensions. However, the way they do so has deeper implications for their functional loss landscape and their interpretability as a scientific tool. Finally, we compare SONODEs and ANODEs on real and synthetic second order dynamical systems. Our results reveal that the inductive biases in SONODEs are beneficial in this setting. Our code is available online at \url{https://github.com/a-norcliffe/sonode}.

\begin{figure}[t]
    \centering
    \includegraphics[width=\textwidth]{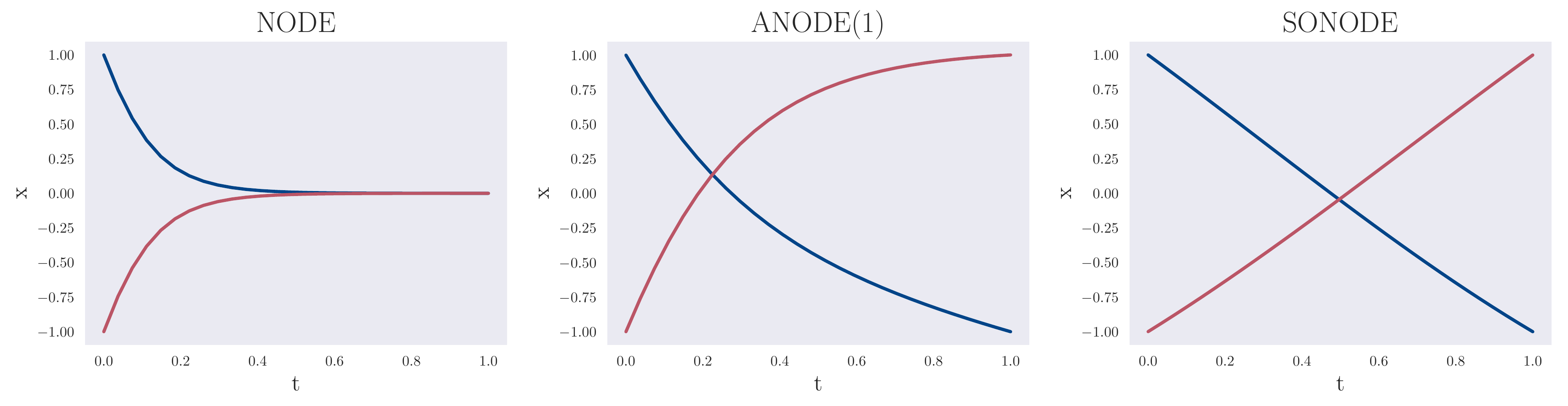}
    \vskip -0.1in
    \caption{Three learnt trajectories from the compact parity experiment ($g_{1d}$ from the original \citep{dupont2019augmented}). NODEs are not able to learn the mapping, ANODE(1) is able to learn it, SONODEs learn the simplest trajectory.}
    \label{fig: g1d}
\end{figure}

\section{Background}

As discussed in the introduction, Neural ODEs (NODEs) can be seen as a continuous variant of ResNet models \cite{he2015deep}, whose hidden state evolves continuously according to a differential equation
\begin{equation}
    \dot{\mathbf{x}} = f^{(v)}(\mathbf{x}, t, \theta_f), \qquad \mathbf{x}(t_0) = \mathbf{X}_0,
\end{equation}
whose velocity is described by a neural network $f^{(v)}$ with parameters $\theta_f$ and initial position given by the points of a dataset $\mathbf{X}_0$. As shown by \citet{chen2018neural}, the gradients can be computed through an abstract adjoint state $\mathbf{r}(t)$, once its dynamics are known.  

Our investigations are mainly focused on Augmented Neural ODEs (ANODEs)~\citep{dupont2019augmented}, which append states $\mathbf{a}(t)$ to the ODE: 
\begin{equation}
\label{eq:anode}
    \mathbf{z} = \begin{bmatrix}
           \mathbf{x} \\
           \mathbf{a} \\
         \end{bmatrix},
    \quad
    \dot{\mathbf{z}} = 
    f^{(v)}(\mathbf{z}, t, \theta_{f}),
    \quad
    \mathbf{z}(t_0)
         = \begin{bmatrix}
           \mathbf{X}_0 \\
           g(\mathbf{X}_0, \theta_{g}) \\
         \end{bmatrix}.
\end{equation}
We note that, unlike the original formulation, we allow for the initial values of the augmented dimensions $\mathbf{a}(t_0)$ to be learned as a function of $\mathbf{x}(t_0)$ by a neural network $g$ with parameters $\theta_g$. For the remainder of the paper, we use the ANODE($D$) notation to signify the use of $D$ augmented dimensions. 

We are almost exclusively concerned with the problem of learning and modelling the behaviour of dynamical systems, given $N+1$ sample points $\mathbf{X}_{t \in T}$, $t=(t_{0},\dots, t_{N}$), from a fixed set of its trajectories at multiple time steps included in the set $T$. For such tasks, we use the mean squared error (MSE) between these points and the corresponding predicted location over all time steps for training the models. For the few toy classification tasks we include, we optimise only for the linear separability of the final positions via the cross-entropy loss function.

\section{Second Order Neural Ordinary Differential Equations}
\label{sec: sonode}

We consider Second Order Neural ODEs (SONODEs), whose initial position $\mathbf{x}(t_{0})$, initial velocity $\dot{\mathbf{x}}(t_{0})$, and acceleration $\ddot{\mathbf{x}}$ are given by
\begin{equation}
\label{eq:sonode_acc_form}
\mathbf{x}(t_{0}) = \mathbf{X}_{0},
\qquad
\qquad
\dot{\mathbf{x}}(t_{0}) = g(\mathbf{x}(t_{0}), \theta_{g}),
\qquad
\qquad
\ddot{\mathbf{x}} = f^{(a)}(\mathbf{x}, \dot{\mathbf{x}}, t, \theta_f),
\end{equation}
where $f^{(a)}$ is a neural network with parameters $\theta_{f}$. Alternatively, SONODEs can be seen as a system of coupled first-order neural ODEs with state $\mathbf{z}(t) = [\mathbf{x}(t), \mathbf{a}(t)]$:
\begin{equation}
\label{eq:sonode_coupled}
    \mathbf{z} = \begin{bmatrix}
           \mathbf{x} \\
           \mathbf{a} \\
         \end{bmatrix},
    \quad
    \dot{\mathbf{z}} = 
    f^{(v)}(\mathbf{z}, t, \theta_{f})
    =
    \begin{bmatrix}
           \mathbf{a} \\
           f^{(a)}(\mathbf{x}, \mathbf{a}, t, \theta_{f}) \\
         \end{bmatrix},
    \quad
    \mathbf{z}(t_0)
         = \begin{bmatrix}
           \mathbf{X}_{0} \\
           g(\mathbf{X}_{0}, \theta_{g}) \\
         \end{bmatrix}.
\end{equation}

This formulation makes clear that SONODEs are a type of ANODE with constraints on the structure of $f^{(v)}$, and offers a way to reuse NODE's first order adjoint method \citep{chen2018neural} for training, as in previous work \citep{yldz2019ode2vae, massaroli2020dissecting}. However, a pair of questions remain about the optimisation of SONODEs: firstly, what is the ODE that the second order adjoint follows? And, consequently, how does the second order adjoint sensitivity method compare with first order adjoint-based optimisation? To address these questions, we show how the adjoint sensitivity method can be generalised to SONODEs.

\begin{proposition}
\label{prop: second_order_adjoint}
The adjoint state $\mathbf{r}(t)$ of SONODEs follows the second order ODE
\begin{equation}
\label{eqn: second_order_adjoint_maint_text_ode}
\begin{aligned}
    \ddot{\mathbf{r}} &= \mathbf{r}^{T}\frac{\partial f^{(a)}}{\partial \mathbf{x}}
    -\dot{\mathbf{r}}^{T}\frac{\partial f^{(a)}}{\partial \dot{\mathbf{x}}}
    -\mathbf{r}^{T}\frac{d}{dt}\Biggr(\frac{\partial f^{(a)}}{\partial \dot{\mathbf{x}}}\Biggr)
\end{aligned}
\end{equation}
\end{proposition}

The proof and boundary conditions for this ODE is given in Appendix~\ref{app: adjoint_method}. As an additional contribution, we include an alternative proof to those of \citet{chen2018neural} and \citet{pontryagin2018mathematical} for the first order adjoint. Given that the dynamics of the abstract adjoint vector are known, its state at all times $t$ can be used to train the parameters $\theta_f$ using the integral
\begin{equation}
\label{eq:sonode_grad}
    \frac{dL}{d\theta_{f}} = -\int_{t_{n}}^{t_{0}}\mathbf{r}^{T}\frac{\partial f^{(a)}}{\partial \theta_{f}}dt,
\end{equation}
where $L$ denotes the loss function and $t_{n}$ is the timestamp of interest. The gradient with respect to the parameters of the initial velocity network, $\theta_g$, can be found in Appendix \ref{app: adjoint_method}. To answer the second question, we compare this gradient against that obtained through the adjoint of the first order coupled ODE from Equation (\ref{eq:sonode_coupled}). 

\begin{proposition}
\label{prop: adjoints_are_equivalent}
The gradient of $\theta_f$ computed through the adjoint of the coupled ODE from (\ref{eq:sonode_coupled}) and the gradient from (\ref{eq:sonode_grad}) are equivalent. However, the latter requires at least as many matrix multiplications as the former. 
\end{proposition}

This result motivates the use of the first order coupled ODE as it presents computational advantages. The proof in Appendix~\ref{app: adjoint_method} shows that this is due to the dynamics of the adjoint from the coupled ODE, which contain entangled representations of the adjoint. This is in contrast to the disentangled representation in Equation (\ref{eqn: second_order_adjoint_maint_text_ode}), where the adjoint state and velocity are separated. It is the entangled representation that permits the faster computation of the gradients for the coupled ODE. We will see in Section \ref{sec: anodes_interpretability} that entangled representations in ANODEs are a reoccurring phenomenon, and their effects are not always beneficial, as in this case. We use the first order ODE optimisation for the remainder of our experiments.

\section{Properties of SONODEs}
\label{sec: node_problems}

In this section, we analyse certain properties of SONODEs and illustrate them with toy examples.

\subsection{Generalised parity problem}

It is known that unique trajectories in NODEs cannot cross at the same time \citep{dupont2019augmented, massaroli2020dissecting}. We extend this to higher order Initial Value Problems (IVP). Proofs are presented in Appendix \ref{app: ode_proofs}.

\begin{proposition}
\label{prop: no_crossing_different}
For a k-th order IVP, if the k-th derivative of $\mathbf{x}$ is Lipschitz continuous and has no explicit time dependence, then unique phase space trajectories cannot intersect at an angle. Similarly, a single phase space trajectory cannot intersect itself at an angle.
\end{proposition}

While this shows SONODE trajectories cannot cross in phase space, they can cross in real space if they have different velocities. To illustrate this, we introduce a generalised parity problem, an extension to $D$ dimensions of the $g_{1d}$ function from \citet{dupont2019augmented}, which maps $\mathbf{x} \to -\mathbf{x}$. We remark that SONODEs should be able to learn a parity flip in any number of dimensions, with a trivial solution
\begin{equation}
\label{eqn: sonode_parity_solution}
f^{(a)}(\mathbf{x}, \dot{\mathbf{x}}, t, \theta_{f}) = 0,
\qquad
\qquad
g(\mathbf{x}(t_{0}), \theta_{g}) = -\frac{2}{t_{N}-t_{0}}\mathbf{x}(t_{0})
\end{equation}
This is equivalent to all points moving in straight lines through the origin to $-\mathbf{x}(t_{0})$. We first visualise the learnt transformation in the one dimensional case (Figure \ref{fig: g1d}), for points initialised at $\pm1$. SONODEs learn the simplest trajectories for this problem.

\begin{wrapfigure}{li}{0.4\textwidth}
    \centering
    \includegraphics[width=0.4\textwidth]{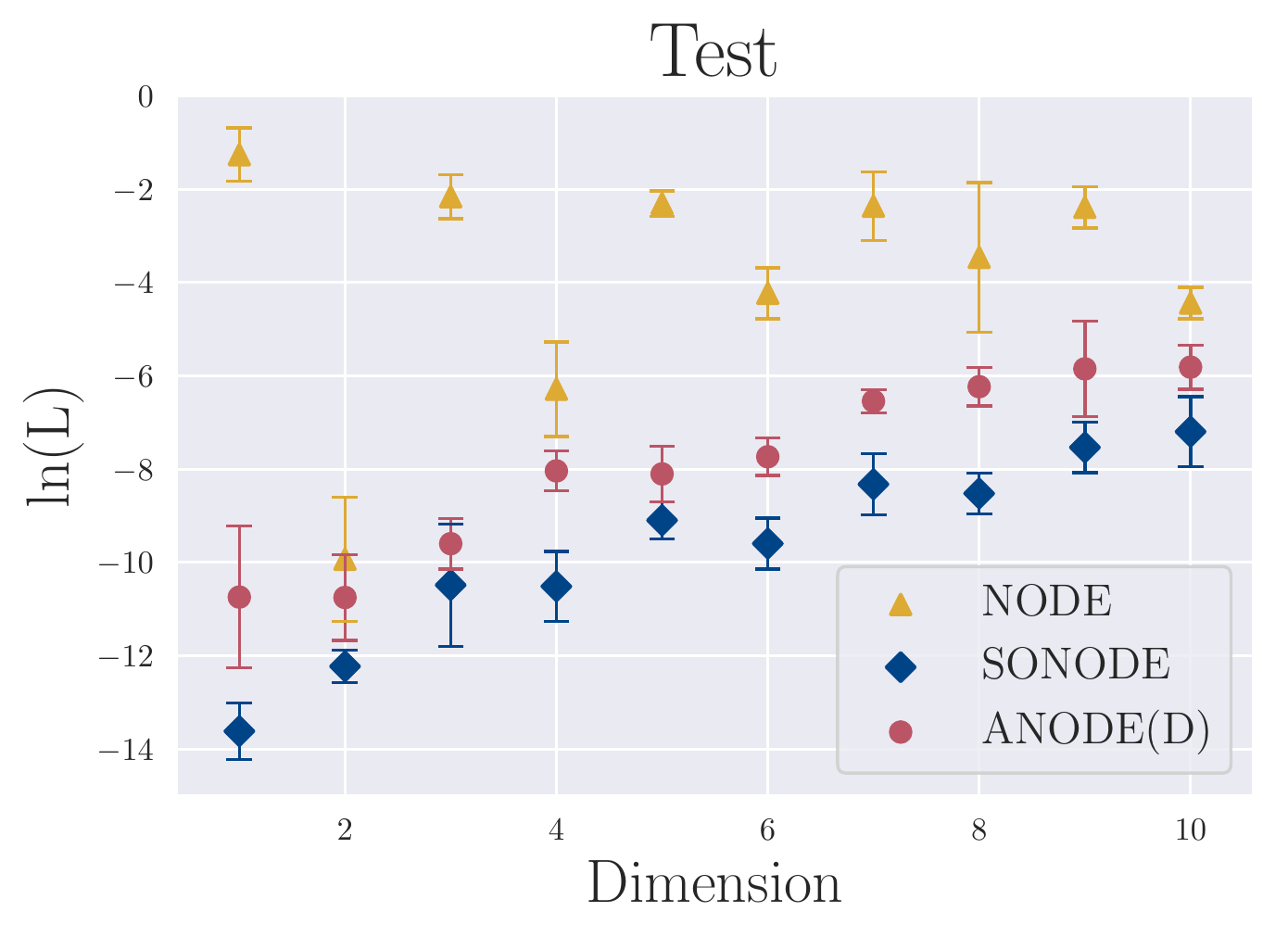}
    \vspace{-10pt}
    \caption{The logarithm of the loss in each dimension for the generalised parity problem. SONODE has the lowest loss, while the NODE loss generally oscillates between dimensions as predicted.}
    \label{fig: parity_multi_dimensions}
\end{wrapfigure}

For higher dimensions, we first remark that NODEs are able to produce parity flips for even dimensions by pairing off the dimensions and performing a $180^{\circ}$ rotation in each pair. This solution does 
not apply to odd-dimensional cases because there is always an unpaired dimension that is not rotated.
In addition to the dimensional-parity effect, as volume increases exponentially with the dimensionality, the density exponentially decreases (given the number of points in the dataset remains constant). This makes it easier to manipulate the points without trajectories crossing, and so, it is expected that the problem will become easier for NODEs as dimensionality increases.

In Figure \ref{fig: parity_multi_dimensions}, we investigate parity flips in higher dimensions, using 50 training points and 10 test points, randomly generated between -1 and 1 in each dimension. For NODEs, as predicted, the loss oscillates over dimensions and, for odd dimensions, the loss decreases with the number of dimensions. ANODEs perform better than NODES, especially in odd dimensions, where it can rotate the unpaired dimension through the additional space. SONODEs have the lowest loss in every generalisation, which can be associated with the existence of the trivial solution in any number of dimensions, given by Equation (\ref{eqn: sonode_parity_solution}).

\subsubsection{Nested n-spheres}
\begin{figure}[b]
    \centering
    \includegraphics[width=0.9\textwidth]{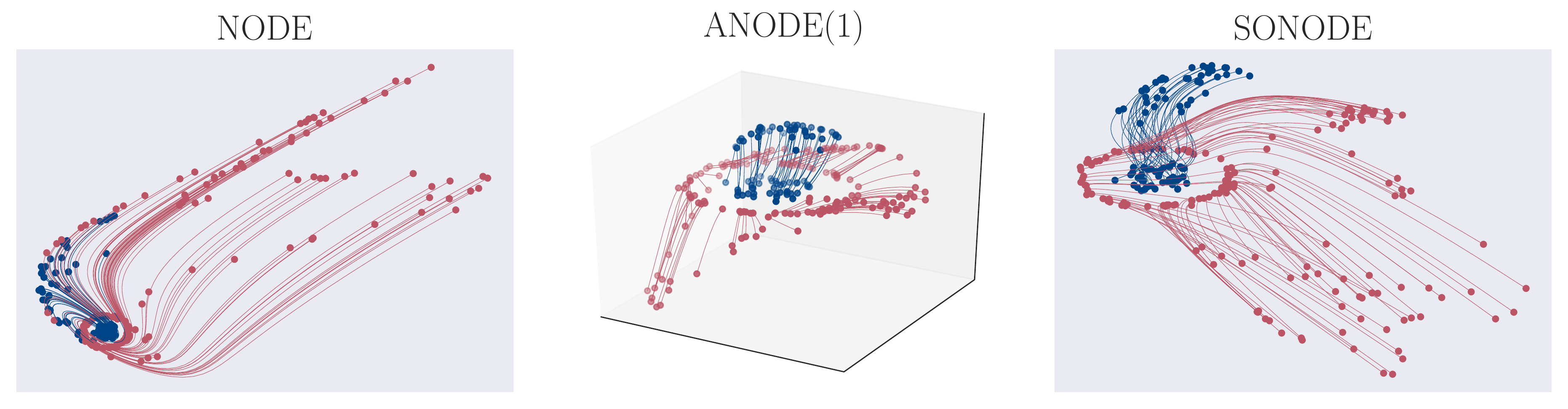}
    \caption{The trajectories learnt by NODEs, ANODEs and SONODEs for the nested-n-spheres problem in 2D. NODEs preserves the topology so the blue region cannot escape the red region. ANODEs, as expected, use the third dimension to separate the two regions. For SONODEs the points pass through each other in real space.}
    \label{fig: g}
\end{figure}

\citet{dupont2019augmented} prove that a transformation under NODEs has to be a homeomorphism, preserving the topology of the input space, and as such, they cannot learn certain transformations. Similarly to ANODEs, SONODEs avoid this problem.

\begin{proposition}
SONODEs are not restricted to homeomorphic transformations in real space.
\end{proposition}

The proof can be found in Appendix \ref{app: sonode_not_homeomorphism}. To illustrate this, we perform an experiment on the nested n-spheres problem \citep{dupont2019augmented}, (the name is taken from \cite{massaroli2020dissecting}, originally called $g$ function \cite{dupont2019augmented}), where the elements of the blue class are surrounded by the elements of the red class (Figure \ref{fig: g}) such that a homeomorphic transformation in that space cannot linearly separate the two classes. As expected, only ANODEs and SONODEs can learn a mapping.

\section{Second order behaviour in SONODEs and ANODEs}
\label{sec: anodes_learn_2nd_order}

Previously, the benefits of ANODEs were attributed only to the extra space they have in which to move \citep{dupont2019augmented}. However, in this section, we show that coupled first order ODEs, such as ANODEs, are also able to represent higher-order order behaviour. Additionally, we study the functional forms ANODEs can use to learn this. Unless stated we consider ANODEs in their original formulation where $\mathbf{a}(t_{0})=0$.

\subsection{How do ANODEs learn second order dynamics?}
\label{sec: how_anodes_learn_2nd}
Consider a SONODE as in Equation (\ref{eq:sonode_acc_form}). Similarly to the coupled ODE from Equation (\ref{eq:sonode_coupled}), ANODEs can represent this if the state, $\mathbf{z} = [\mathbf{x},\mathbf{a}]$, is augmented such that $\mathbf{a}$ has the same dimensionality as $\mathbf{x}$:
\begin{equation}
\label{eqn: anode_learn_2nd_order}
\mathbf{z}(t_{0}) = \begin{bmatrix}
    \mathbf{x}(t_{0})\\
    0\\
\end{bmatrix},
\quad
\dot{\mathbf{z}}
=
\begin{bmatrix}
    \mathbf{a} + \dot{\mathbf{x}}(t_{0})\\
    f^{(a)}(\mathbf{x}, \dot{\mathbf{x}}, t, \theta_{f})\\
\end{bmatrix}
=
\begin{bmatrix}
    \mathbf{a} + g(\mathbf{x}(t_{0}), \theta_{g})\\
    f^{(a)}(\mathbf{x}, \mathbf{a}+g(\mathbf{x}(t_{0}),\theta_{g}), t, \theta_{f})\\
\end{bmatrix},
\end{equation}
where $\mathbf{a}$ differentiates to the acceleration and, because $\mathbf{a}(t_{0})=0$, the initial velocity is added to it to obtain the correct dynamics. Generalising this, it is clear to see how ANODEs can also learn $k$-th order ODEs, by splitting up the augmented part $\mathbf{a}$ into $k-1$ vectors with the same dimensionality as $\mathbf{x}$. However, if ANODEs were to learn higher order dynamics this way, $\mathbf{x}(t_{0})$ is required as an input, just as in data-controlled neural ODEs \cite{massaroli2020dissecting}. To show this is not usually the case, we let ANODE(1) learn two 1D functions at the same time with a shared ODE, using the same set of parameters, but different initial conditions. Specifically, we consider two damped harmonic oscillators
\begin{equation}
x_{1}(t) = e^{-\gamma t}\sin(\omega t),
\qquad
\qquad
x_{2}(t) = e^{-\gamma t}\cos(\omega t)
\end{equation}
where $\gamma$ can be zero so that there is no decay. 

SONODEs can learn these using the functional form
\begin{equation}
\label{eqn: sonode_double_func}
f^{(a)}(x, \dot{x}, t, \theta_{f}) = -(\omega^{2} +\gamma^{2})x - 2\gamma\dot{x},
\qquad
\qquad
g(x(0),\theta_{g}) = -(\omega+\gamma)x(0) + \omega
\end{equation}

\begin{wrapfigure}{li}{0.4\textwidth}
    \begin{center}
    \vspace{-15pt}
    \includegraphics[width=0.4\textwidth]
    {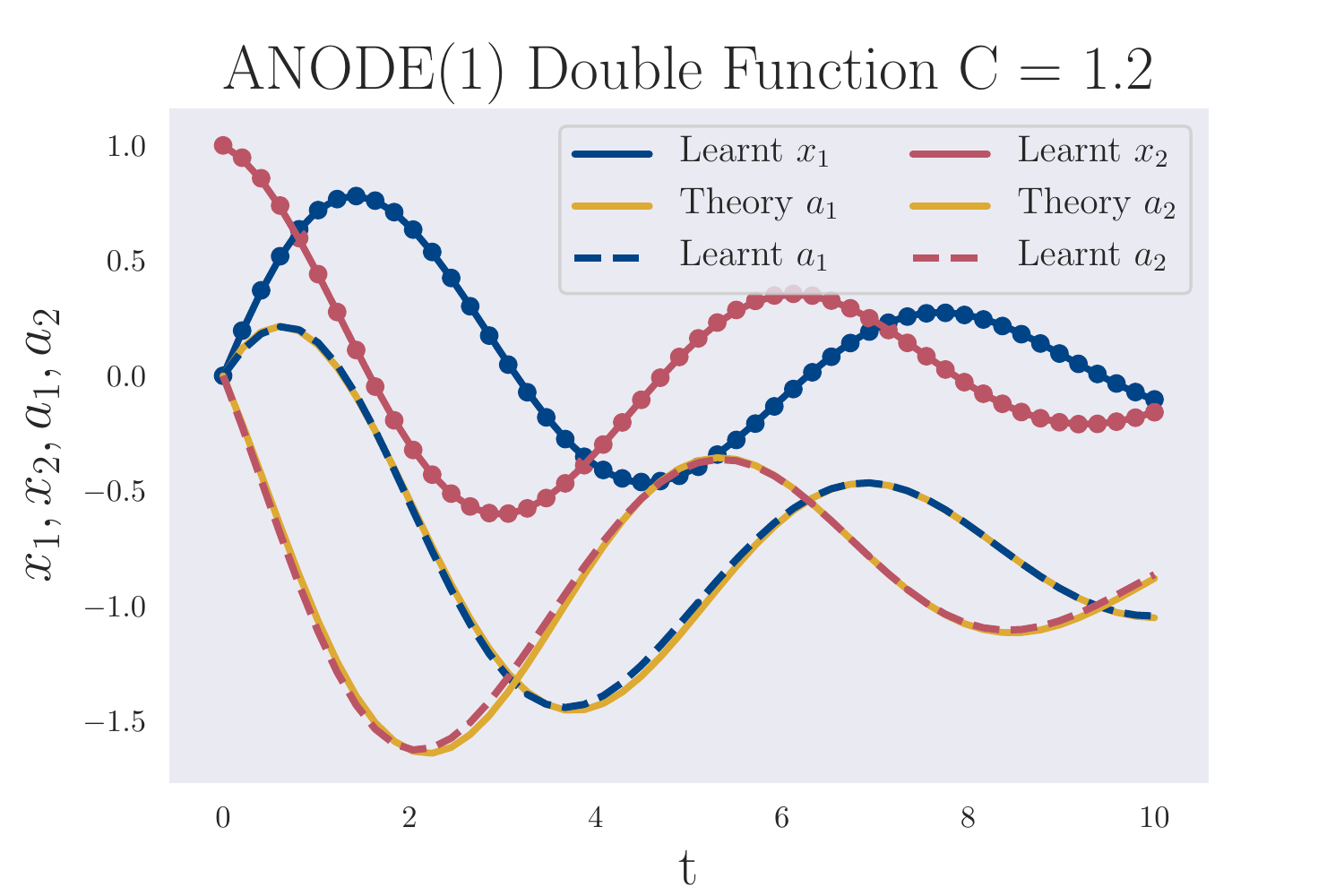}
    \vspace{-10pt}
    \caption{ANODE(1) learning two functions using the same parameters, for $\omega=1$ and $\gamma=0.1667$. The real trajectories are going through their sampled data points. Augmented trajectories are plotted over their theoretical trajectories given by Equation (\ref{eqn: anode_learning_double_func}) for $C=1.2$.}
    \vspace{-30pt}
    \label{fig: anodes_double_func}
    \end{center}
\end{wrapfigure}

It is not immediately obvious how ANODEs could solve this, especially if they follow Equation (\ref{eqn: anode_learn_2nd_order}), where $x(t_{0})$ is needed as an input to determine $\dot{x}(t_{0})$. However, Figure \ref{fig: anodes_double_func} shows that ANODEs are able to fit the two functions in the same training session. We observe that ANODEs approximate a solution of the form:
\begin{equation}
\label{eqn: anode_learning_double_func}
\begin{bmatrix}
\dot{x}\\
\dot{a}\\
\end{bmatrix}
=
\begin{bmatrix}
C a -\omega x - \gamma x + \omega\\
\omega a - \gamma a  -\frac{1}{C}(2\omega^{2}x+\gamma\omega - \omega^{2})\\
\end{bmatrix}
\end{equation}
Using $a(0)=0$, this gives the correct ODE and initial conditions in Equation (\ref{eqn: sonode_double_func}), for all finite, non-zero $C$.

We remark that the state $x$ and the augmented dimension $a$ are entangled in the velocity of the state and $\dot{x} \neq a$. This example gives an intuition about the way ANODEs can learn second order behaviour through an ODE as in Equation (\ref{eqn: anode_learning_double_func}). We now formalise this intuition and give a general expression: \newline

\begin{proposition}
\label{prop: anode_general_form}
The general form ANODEs learn second order behaviour is given by:
\begin{equation}
\begin{bmatrix}
\dot{\mathbf{x}}\\
\dot{\mathbf{a}}\\
\end{bmatrix}
=
\begin{bmatrix}
F(\mathbf{x},\mathbf{a},t,\theta_{F})\\
G(\mathbf{x},\mathbf{a},t,\theta_{G})\\
\end{bmatrix},
\qquad
G = \left( \frac{\partial F}{\partial \mathbf{a}^{T}}\right)_{\text{left}}^{-1}
\left(
f^{(a)} - \frac{\partial F}{\partial \mathbf{x}^{T}}F - 
\frac{\partial F}{\partial t}
\right)
\end{equation}
\end{proposition}

This result is derived in Appendix \ref{app: anode_learn_second_order}. It shows that SONODEs and ANODEs learn second order dynamics in different ways. ANODEs learn an abstract function $F$ that at $t_0$ is equal to the initial velocity, and another function $G$ that couples to $F$ giving it the right acceleration. In contrast, SONODEs are constrained to learn the acceleration and initial velocity directly. This also leads to several useful properties that we investigate next. 

\subsection{Minimal augmentation}

The first property we analyse is called \textit{minimal augmentation}. It refers to the fact that ANODEs can learn second order dynamics even when the number of extra dimensions is less than the dimensionality of the real space.

\begin{corollary}
When the system from Proposition \ref{prop: anode_general_form} is overdetermined (i.e. $\text{dim}(\mathbf{a}) < \text{dim}(\mathbf{x})$) and has a solution, the Moore-Penrose left pseudo-inverse produces that solution, given by $G$. If no solution exists, $G$ is the best least-squares approximation.  
\end{corollary}

\begin{figure}[h]
    \centering
    \includegraphics[width=0.9\textwidth]{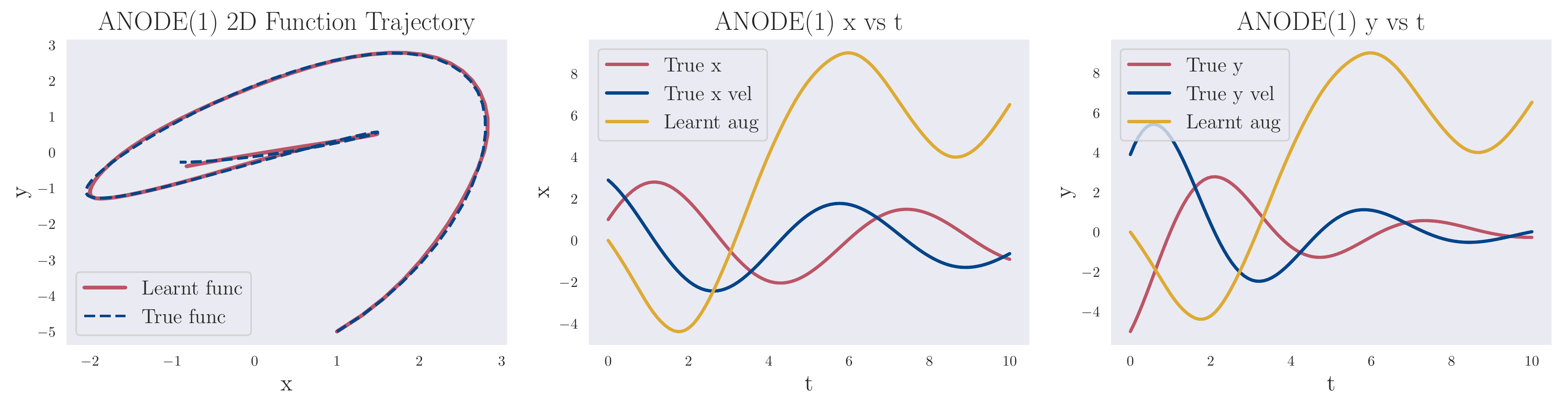}
    \vspace{-10pt}
    \caption{ANODE(1) learning a 2D second order function. ANODE(1) is able to learn the function, showing that it does not necessarily need double the dimensions to learn second order.}
    \label{fig: anode_1_2dfunc}
\end{figure}

In effect, ANODE is learning a system of linear equations parametrised by deep neural networks. To learn second order dynamics with minimal augmentation, it must learn an overdetermined linear system allowing a solution. Depending on the form of $\ddot{\mathbf{x}}$, it is possible that an $F$ with explicit $\mathbf{a}$ dependence that produces a degenerate system like this could be learned. In turn, this would allow a complementary $G$ to be learned. In fact, systems like this can naturally arise when the dynamics are latent and lower-dimensional and many of the observed dimensions become redundant. For instance, two spatial dimensions suffice for a pendulum moving in a plane of the 3D space.    

However, even if an overdetermined system allowing a solution could not be learned due to the additional constraints acting on $F$, the left Moore-Penrose pseudo-inverse from Proposition \ref{prop: anode_general_form} would still produce a $G$ that is a best least-squares approximation. If the matrix $A = \frac{\partial F}{\partial \mathbf{a}^{T}}$ has full rank, then the left inverse is given by $(A^TA)^{-1}A^T$. In general, the closer dim$(\mathbf{a})$ gets to dim$(\mathbf{x})$, the better this approximation will be. 

To demonstrate minimal augmentation, we consider a two dimensional second order ODE, whose starting conditions and respective $\omega$'s and $\gamma$'s were chosen randomly such that
\begin{equation}
\label{eq:2d_function}
    \begin{bmatrix}
    \ddot{x} \\
    \ddot{y} \\
    \end{bmatrix}
    =
    \begin{bmatrix}
    -(\omega_{x}^{2} + \gamma_{x}^{2})x -2\gamma_{x}\dot{x} \\
    -(\omega_{y}^{2} + \gamma_{y}^{2})y -2\gamma_{y}\dot{y} \\
    \end{bmatrix},
\qquad
    \begin{bmatrix}
    x \\
    y \\
    \end{bmatrix}
    =
    \begin{bmatrix}
    e^{-0.1t}(3\sin(t)+\cos(t)) \\
    e^{-0.3t}(2\sin(1.2t)-5\cos(1.2t)) \\
    \end{bmatrix}
\end{equation}
ANODE(1) is able to learn this function as shown in Figure \ref{fig: anode_1_2dfunc}. Moreover, the augmented dimension trajectory differs greatly from the velocity of the ODE in either of the two spatial dimensions.

\subsection{Interpretability of ANODEs}
\label{sec: anodes_interpretability}

The result from Proposition \ref{prop: anode_general_form} also raises the issue of how \textit{interpretable} ANODEs are. For example, when investigating the dynamics of physical systems it is useful to know the force equation. This is straightforward with SONODEs, which directly learn the acceleration as a function of position, velocity and time. However, ANODEs learn the dynamics through an abstract alternative ODE where the state and augmented dimensions are entangled. This is similar to the widely studied problem of entangled representations~\citep{Higgins2017betaVAELB, Mathieu2019DisentanglingDI, bengio2012representation}. 

We then train both ANODE(2) and SONODE to learn the dynamics of the ODE from Equation (\ref{eq:2d_function}), and provide them both with the correct initial velocity. Figure \ref{fig: interpretability} shows the results for two different runs for both models. Though ANODE(2) is able to learn the true trajectory in real space, the augmented trajectories differ greatly from the true velocity of the underlying ODE. In contrast, SONODE learns the correct velocity for both runs. This simple experiment confirms that ANODEs might not be a suitable investigative tool for scientific applications, where the physical interpretability of the results is important.

\subsection{The functional loss landscape}
The functional forms the two models converge to in  Figure \ref{fig: interpretability} are not a coincidence. Proposition \ref{prop: anode_general_form} also has deeper implications for the ANODE's \textit{(functional) loss landscape} when learning second order dynamics. Please refer to Appendix \ref{app: anode_learn_second_order} for the proofs of the following results.  

\begin{proposition}
\label{prop: anode_infinity}
There are an infinity of (non-trivial) functional forms ANODEs can learn that model the true second order dynamics in real space.
\end{proposition}

This means that there is an infinite number of functions ANODEs can approximate and obtain a zero loss. This suggests that an infinite number of global minima, representing different functions, may exist in the loss landscape of ANODEs. In contrast, we show that the second order constraints imposed on SONODE enforce that any global minima in its loss landscape approximate the same function --- the acceleration and, in some cases, the initial velocity. 

\begin{proposition}
\label{prop: sonode_unique}
There is a unique functional form SONODEs can learn that models the true second order dynamics in real space.
\end{proposition}

\begin{figure}
    \centering
    \includegraphics[width=\textwidth]{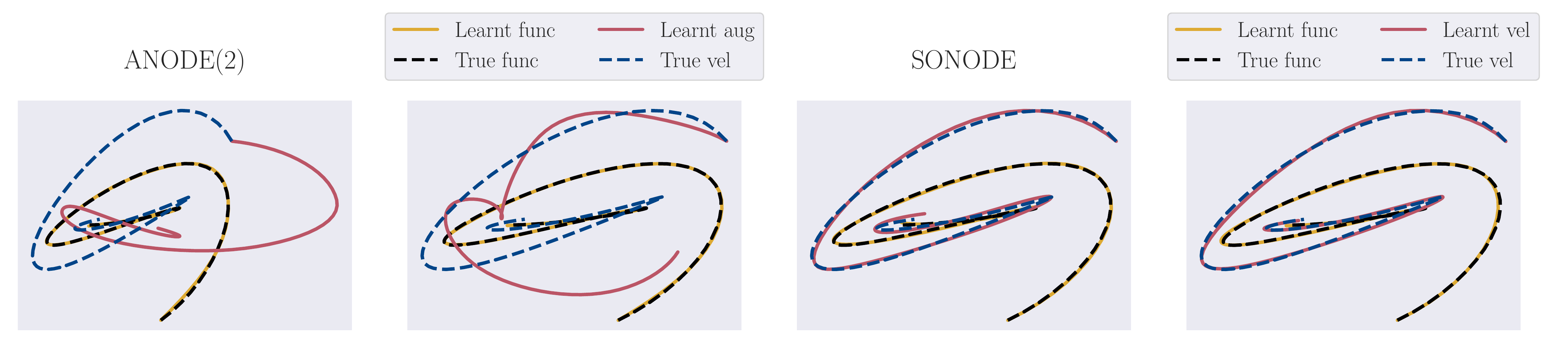}
    \vspace{-8pt}
    \caption{ANODEs and SONODEs successfully learn the trajectory in real space of a 2D ODE for two different random initialisations. However, the augmented trajectories of ANODE are in both cases widely different from the true velocity of the ODE. In contrast, SONODE converges in both cases to the true ODE.}
    \vspace{-10pt}
    \label{fig: interpretability}
\end{figure}

 This is confirmed by our experiment from the previous section, where ANODE always converges to another augmented trajectory for each random initialisation (only two shown in the Figure \ref{fig: interpretability}), while SONODE always converges to the correct underlying ODE.

\section{Experiments on second order dynamics}
\label{sec: experiments}

\begin{wrapfigure}{l}{0.35\textwidth}
    \begin{center}
    \vspace{-15pt}
    \includegraphics[width=0.34\textwidth]{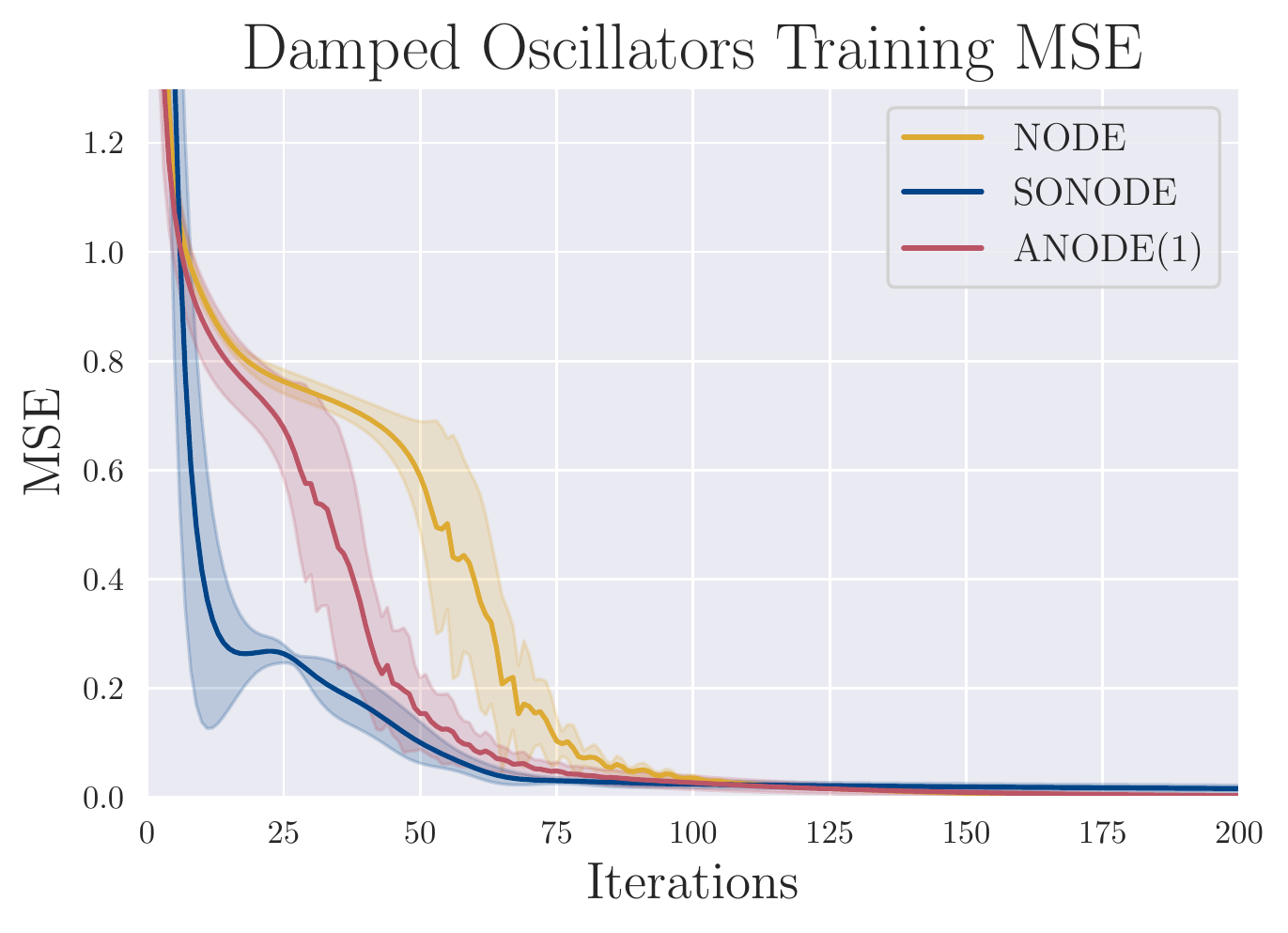}
    \caption{NODE, ANODE(1) and SONODE training on harmonic oscillators. SONODEs already have the second order behaviour built in as an architectural choice, so they are able to learn the dynamics in fewer iterations.}
    \vspace{-45pt}
    \label{fig: damped_oscillators_loss}
    \end{center}
\end{wrapfigure}

To test our above predictions, we perform an extensive comparison of ANODE and SONODE on a set of more challenging real and synthetic modelling tasks. These experiments provide further evidence for the described theoretical findings. Additional experimental details regarding the models and additional results are given in Appendix \ref{app: experimental_setup}.

\subsection{Synthetic harmonic oscillators and noise robustness} \label{sec:synthetic}

\paragraph{Harmonic oscillator} The most obvious application of SONODEs is on dynamical data from classical physics. This was tested by looking at a damped harmonic oscillator $\ddot{x} = -(\omega^{2}+\gamma^{2})x - 2\gamma\dot{x}$ with
$\gamma = 0.1$ and $\omega = 1$ on 30 random pairs of initial positions and velocities. These were each evolved for 10 seconds, using one hundred evenly spaced time stamps. The loss depended on both position and velocity explicitly, therefore the models used the state $\mathbf{z}=[x, v]$ with the option of augmentation for ANODEs. NODEs and ANODEs learnt a general $\dot{\mathbf{z}}$, whereas SONODEs are given $\dot{\mathbf{z}} = [v, f^{(a)}]$ and only learn $f^{(a)}$. SONODEs leverage their inductive bias and converge faster than the other models. Note that, all models were able to reduce the loss to approximately zero, as shown in Figure \ref{fig: damped_oscillators_loss}.

\paragraph{Noise robustness} We tested the models' abilities to learn a sine curve in varying noise regimes. The models were trained on fifty training points in the first ten seconds of $x=\sin(t)$, and then tested with ten points in the next five seconds. The train points had noise added to them, drawn from a normal distribution $\mathcal{N}(0,\sigma^2)$ for different standard deviations $\sigma = (0,0.1,0.2,\dots, 0.7$). The results presented in Figure \ref{fig: noise_robustness} show that SONODEs are more robust to noise. 

\begin{figure}[h]
    \centering
    \includegraphics[width=0.98\textwidth]{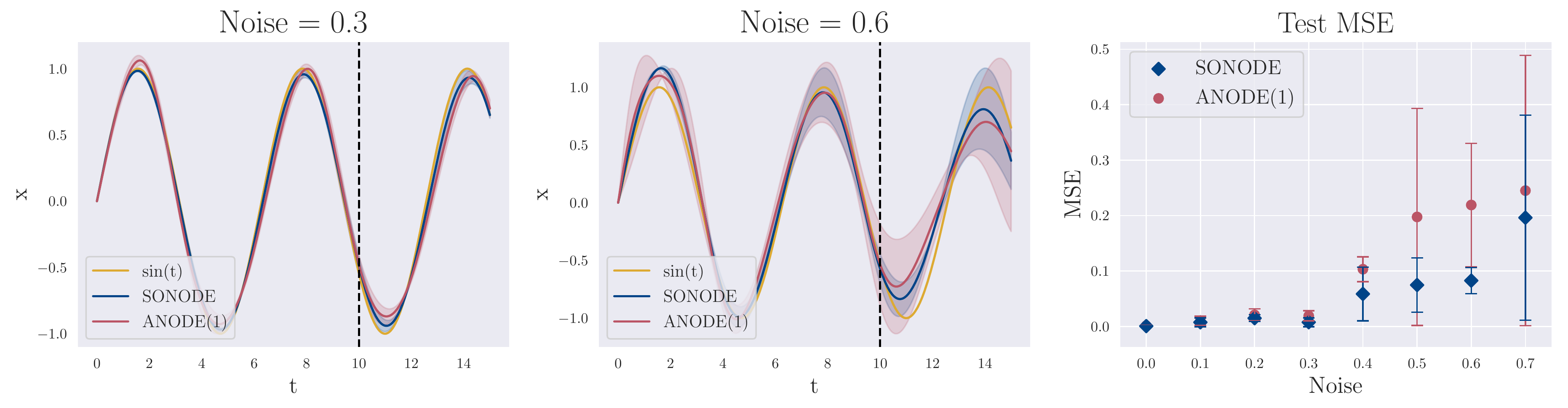}
    \vspace{-10pt}
    \caption{How SONODEs and ANODEs perform learning a sine curve in different noise regimes. The dotted line separates training and testing regimes. SONODEs are able to extrapolate better than ANODEs because they are forced to learn second order dynamics, and therefore are less likely to overfit the training points.}
    \label{fig: noise_robustness}
\end{figure}

\subsection{Experiments on real-world dynamical systems}
\label{sec: planes}

\paragraph{Airplane vibrations} The dataset \cite{noel2017f} concerns real vibrations measurements of an airplane. A shaker was attached underneath the right wing, producing an acceleration $a_{1}$. Additional accelerations at different points were measured including $a_{2}$, which was examined in this experiment, the acceleration on the right wing, next to a non-linear interface of interest. This is a higher order system, therefore it pertains to be a challenging modelling task. The results presented in Figure \ref{fig: f16} show that while both methods can model the dynamics reasonably well, ANODEs perform marginally better. We conjecture that this is due to ANODEs not being restricted to second order behaviour, allowing them to partially access higher order dynamics. We test this conjecture in Appendix \ref{app: plane_part_2}.

\begin{figure}[h]
    \centering
    \includegraphics[width=0.98\textwidth]{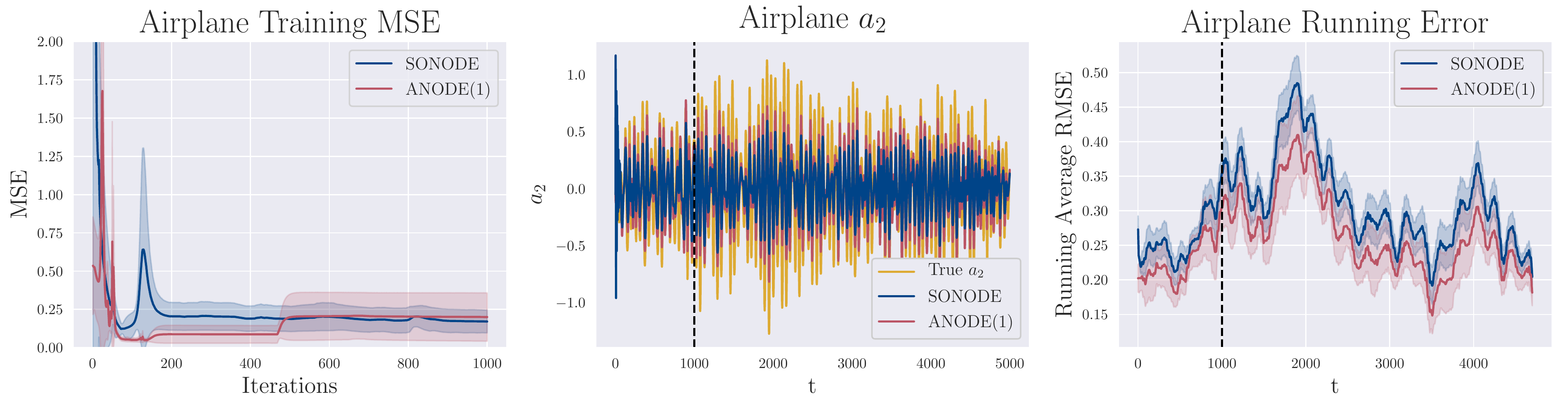}
    \vspace{-10pt}
    \caption{ANODE(1) and SONODE on the Airplane Vibrations dataset. ANODEs are able to perform slightly better than SONODEs because they are able to access higher order dynamics. The models were trained on the first 1000 timestamps and then extrapolated to the next 4000.}
    \label{fig: f16}
\end{figure}

\paragraph{Silverbox oscillator} The Silverbox dataset \cite{6669201} is an electronic circuit resembling a Duffing Oscillator, with input voltage $V_{1}(t)$ and measured output $V_{2}(t)$. The non-linear model Silverbox represents is $\ddot{V}_{2} = a\dot{V}_{2} + bV_{2} + cV_{2}^{3} + dV_{1}$. To account for this, all models included a $V_{2}^{3}$ term. 
The results can be seen in Figure \ref{fig: silverbox}. On this second order system, SONODEs extrapolate better than ANODEs and are able to capture the increase in the amplitude of the signal exceptionally well. 

\begin{figure}[h]
    \centering
    \includegraphics[width=0.98\textwidth]{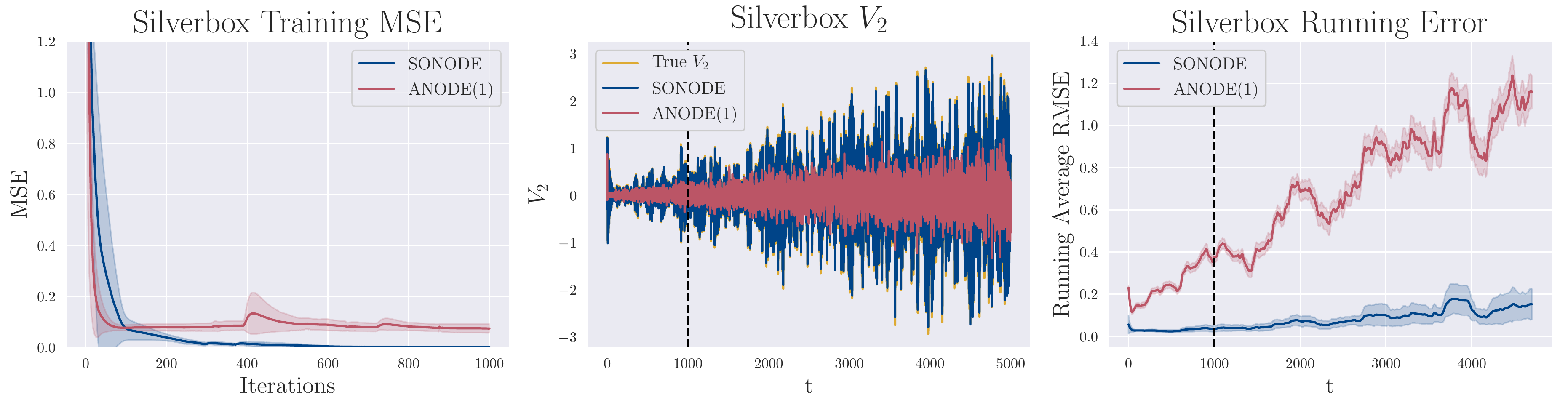}
    \vspace{-10pt}
    \caption{ANODE(1) and SONODE on the Silverbox dataset. SONODEs are able to reduce the loss faster and to a lower value than ANODEs, as expected when second order behaviour is built in. The models were trained on the first 1000 timestamps and extrapolated to the next 4000.}
    \vspace{-10pt}
    \label{fig: silverbox}
\end{figure}

\section{Discussion and related work}

\paragraph{SONODEs vs ANODEs} SONODEs can be seen as a special case of ANODEs, whose phase space dynamics are restricted to model second order behaviour. We believe that for tasks where the trajectory is unimportant, and performance depends only on the endpoints (such as classification), ANODEs might perform better because they are unconstrained in how they use their capacity (see Appendix \ref{app: mnist}). In contrast, we expect SONODEs to outperform ANODEs both in terms of accuracy and convergence rate on time series data whose underlying dynamics is assumed (or known) to be second order. In this setting, SONODEs have a unique functional solution and fewer local minima compared to ANODEs. At the same time, they have higher parameter efficiency since $\dot{\mathbf{x}}=\mathbf{v}$ requires no parameters, so all parameters are in the acceleration. Finally, we expect SONODEs to be more appropriate for application in the natural sciences, where second order dynamics are common and it is useful to recover the force equation. 

\paragraph{Second Order Models} Concurrent to our work, SONODEs have been briefly evaluated on MNIST by \citet{massaroli2020dissecting} as part of a wider study on Neural ODEs. In contrast, our study is focused on the theoretical understanding of second order behaviour. At the same time, our investigations are largely based on learning the dynamics of physical systems rather than classification tasks. Second order models have also been considered in Graph Differential Equations~\citep{poli2019graph} and ODE$^2$VAE \citep{yldz2019ode2vae}. 

\paragraph{Physics Based Models} In the same way SONODEs assert Newtonian mechanics, other models have been made to use physical laws, guaranteeing physically plausible results, in discrete and continuous cases. \citet{lutter2019deep} apply Lagrangian mechanics to cyber-physical systems, while \citet{greydanus2019hamiltonian} and \citet{zhong2019symplectic} use Hamiltonian mechanics to learn dynamical data. 

\section{Conclusion}

In this paper, we took a closer look at how Neural ODEs (NODEs) can learn second order dynamics. In particular, we considered Second Order NODEs (SONODEs), a model constructed with this inductive bias in mind, and the more general class of Augmented Neural ODEs (ANODEs). We began by shedding light on the optimisation of SONODEs by generalising the adjoint sensitivity method from NODEs and comparing it with the training procedure of the equivalent coupled ODE. We also studied the theoretical properties of SONODEs and how they manifest in modelling toy problems.

We showed that, despite lacking the physics-based inductive biases of SONODEs, ANODEs are flexible enough to learn second order dynamics in practice. However, we also demonstrated, analytically and empirically, that they do this by learning to approximate an abstract coupled ODE where the state and augmented dimensions become entangled in the velocity. We proved that this has implications for interpretability in scientific applications as well as the `shape' of the loss landscape. Our experiments on synthetic and real second order dynamical systems validate these concerns and reveal that the inductive biases of SONODE are generally beneficial in this setting. Although this work investigates second order dynamics, the underlying principles of SONODEs can be readily extended to higher orders (a proof-of-principle is given in Appendix~\ref{app: plane_part_2}). This, in turn, allows for modelling richer and more complex behaviour, while retaining the benefits of faster training and better modelling performance.

\clearpage

\section*{Broader Impact}
Neural ODEs are relatively new models and we are yet to see their full potential. We anticipate NODEs will see particular success in time-series data, which have a wide variety of real-world applications. Examples given by \citet{jia2019neural} include the evolution of individuals' medical records and earthquake monitoring. \citet{poli2019graph} look at traffic forecasting and \citet{greydanus2019hamiltonian} show how a Neural ODE inspired by Hamiltonian mechanics can be applied to classical physics. Our work concerns Second Order Neural ODEs which can also be applied to classical physics, where Newton's second law describes the forces on an object. 

Our theoretical work was concerned with demonstrating how best to use the adjoint method on SONODEs, and showing how the coupled ODE perspective of ANODEs leads to them being able to learn second order behaviour. Naturally, any impacts from this work will come from the applications of SONODEs.

We directly investigated two potential real-world applications of SONODEs. The Silverbox dataset, an electronic implementation of a damped spring with a non-linear spring constant. This naturally applies to circuits with oscillators, and damped elements, opening new directions to monitor circuits and signals. The dynamics can also be encountered in mechanical systems, including car suspension, which could be used to improve car safety. Note that, in our experiments, we also investigated the task of modelling the vibration dynamics of an aeroplane, which might lead to better and optimal aeroplane designs. Though contributions to civil mechanical engineering such as these have parallel applications in the design of weapons, it is not the case that our investigation expands technological capabilities in such a way as to enable new forms of warfare or to significantly improve current technologies (at this stage.)

As stated, Neural ODEs are relatively new, and we are yet to see their full potential. We anticipate more applications to time series data in the future, which have many positive and negative applications, though at most we should think of our contribution as incremental in this regard and covered by existing institutions and norms.

\begin{ack}


We would like to thank C\u{a}t\u{a}lina Cangea, Jacob Deasy and Duo Wang for their helpful comments. We would like to also thank the reviewers for their constructive feedback and efforts towards improving our paper. The authors declare no competing interests.

\end{ack}


\bibliographystyle{plainnat}
\bibliography{references}

\newpage
\appendix

\section{Phase Space Trajectory Proofs}
\label{app: ode_proofs}
Here we present the proofs for the propositions from Section \ref{sec: node_problems}, concerning a $k$-th order initial value problem.

\begin{lemma}
\label{lem: no_kink}
For a k-th order IVP, where the k-th derivative is Lipschitz continuous, a solution cannot have discontinuities in the time derivative of its phase space trajectory.
\end{lemma}

\begin{proof}
Consider the phase space trajectory $\displaystyle \mathbf{z}(t) = \left[\mathbf{x}(t), \frac{d\mathbf{x}}{dt}(t), ... , \frac{d^{k-1}\mathbf{x}}{dt^{k-1}}(t) \right]$. Let $f$ be the k-th time derivative of $\mathbf{x}(t)$. Then the time derivative of $\mathbf{z}(t)$ is
\[
\frac{d}{dt}
\begin{bmatrix}
\mathbf{x}\\[5pt]
\displaystyle \frac{d\mathbf{x}}{dt}\\
...\\
\displaystyle \frac{d^{k-1}\mathbf{x}}{dt^{k-1}}\\
\end{bmatrix}
=
\begin{bmatrix}
\displaystyle \frac{d\mathbf{x}}{dt}\\[7pt]
\displaystyle \frac{d^{2}\mathbf{x}}{dt^{2}}\\
...\\
f(\mathbf{z})\\
\end{bmatrix}
\]
If for one set of finite arguments, $\mathbf{z}_{1}$, $f(\mathbf{z_{1}})$ is also finite, then because the gradients of $f$ are all bounded (due to Lipschitz continuity), for any other finite arguments, $\mathbf{z}_{n}$, $f(\mathbf{z}_{n})$ will remain finite. Now consider $\displaystyle  \frac{d^{k-1}\mathbf{x}}{dt^{k-1}}$, its time derivative is $f(\mathbf{z}(t))$, which is finite for all finite $\mathbf{z}$. Therefore, $\displaystyle  \frac{d^{k-1}\mathbf{x}}{dt^{k-1}}$, can't have discontinuities with a finite derivative, and also must be finite for finite $\mathbf{z}$. Now consider $\displaystyle  \frac{d^{k-2}\mathbf{x}}{dt^{k-2}}$, its time derivative is finite for all finite $\mathbf{z}$, and therefore it can't have discontinuities and also must be finite for all finite $\mathbf{z}$. This line of argument continues up to $\mathbf{x}$. The state $\mathbf{x}$ and all of its time derivatives up to the $k$-th have no discontinuities and are finite. Therefore as long as the initial conditions $\mathbf{z}(t_{0})$ are finite, there can be no discontinuities in the time derivative of the phase space trajectory at finite time.
\end{proof}

\textbf{Proposition \ref{prop: no_crossing_different}.} \textit{For a k-th order IVP, if the k-th derivative of $\mathbf{x}$ is Lipschitz continuous and has no explicit time dependence, then unique phase space trajectories cannot intersect at an angle. Similarly, a single phase space trajectory cannot intersect itself at an angle.}

\begin{proof}
Consider two trajectories $\mathbf{z}_{1}(t)$ and $\mathbf{z}_2(t)$ that have different initial conditions $\mathbf{z}_{1}(t_{0}) = \mathbf{h}_1$ and $\mathbf{z}_{2}(t_{0}) = \mathbf{h}_2$. Assume the trajectories cross at a point in phase space at an angle, $\mathbf{z}_1(t_{1}) = \mathbf{z}_{2}(t_{2}) = \mathbf{\tilde{h}}$. If they intersect at an angle, then evolving the two states by a small time $\delta t << 1$, and using the Lipschitz continuity of $f$, meaning that the trajectories cannot have kinks in them (as shown in Lemma \ref{lem: no_kink}), $\mathbf{z}_{1}(t_{1}+\delta t) \neq \mathbf{z}_{2}(t_{2}+\delta t)$. However, if they are at the same point in phase space, then they must have the same k-th order derivative, $f$. All other derivatives are equal, so by evolving the states by the same small time $\delta t << 1$, $\mathbf{z}_{1}(t_{1}+\delta t) = \mathbf{z}_{2}(t_{2}+\delta t)$. There is a contradiction  and therefore the assumption is wrong, unique trajectories cannot cross at an angle in phase space when $f$ is Lipschitz continuous and has no $t$ dependence.

Now consider the single trajectory $\mathbf{z}(t)$. Assume it intersects itself at an angle, at $t_{1}$ and $t_{2}$. Now consider two particles on this trajectory, starting at $t_{1}-\tau$ and $t_{2}-\tau$ such that $t_{2}-\tau > t_{1}$. These two particles have different initial conditions and cross at an angle. However, the above shows that cannot happen. Therefore, the assumption that $\mathbf{z}(t)$ can intersect itself at an angle must be wrong. Trajectories cannot intersect themselves in phase space at an angle.
\end{proof}

Trajectories can, however, feed into each other representing the same particle path at different times. Single phase space trajectories can feed into themselves representing periodic motion. This requires a Lipschitz continuous $f$, and for there to be no explicit time dependence. If there was time dependence then two trajectories can cross at different times, and a trajectory can self intersect. Effectively an additional dimension is added to phase space, which is time. The propositions above would still hold because $\displaystyle  \frac{dt}{dt} = 1$ which is Lipschitz continuous. Therefore, with time included as a phase space dimension, intersections in space are only forbidden if they occur at the same time.

\section{Adjoint Sensitivity Method}
\label{app: adjoint_method}

We present a proof to both the first and second order Adjoint method, using a Lagrangian style approach \citep{stanford_adjoint_derivation, Gholaminejad_2019}. We also prove that when the underlying ODE is second order, using the first order method on a concatenated state, $\mathbf{z} = [\mathbf{x}, \mathbf{v}]$, produces the same results as the second order method but does so more efficiently. All parameters, $\theta$, are time-dependent.

\subsection{First Order Adjoint Method}

Let L denote a scalar loss function, $L = L(\mathbf{x}(t_{n}))$, the gradient with respect to a parameter $\theta$ is
\begin{equation}
\label{eqn: dLdtheta}
    \frac{dL}{d\theta} = \frac{\partial L}{\partial \mathbf{x}(t_{n})^{T}}\frac{d\mathbf{x}(t_{n})}{d\theta}
\end{equation}

The vector $\displaystyle  \frac{\partial L}{\partial \mathbf{x}(t_{n})^{T}}$ is found using backpropagation. For dynamical data the loss will depend on multiple time stamps, there is also a sum over timestamps, $t_{n}$. Therefore $\displaystyle  \frac{d\mathbf{x}(t_{n})}{d\theta}$ is needed. $\mathbf{x}(t_{n})$ follows
\begin{equation}
    \mathbf{x}(t_{n}) = \int_{t_{0}}^{t_{n}}\dot{\mathbf{x}}(t)dt + \mathbf{x}(t_{0})
\end{equation}
subject to
\begin{equation}
    \dot{\mathbf{x}} = f^{(v)}(\mathbf{x}, t,\theta_{f}),
    \qquad
    \qquad
    \mathbf{x}(t_{0}) = s(\mathbf{X}_{0}, \theta_{s})
\end{equation}

where $\mathbf{X}_{0}$ is the data going into the network and is constant. The functions $f^{(v)}$ and $s$ describe the ODE and the initial conditions. Here we allow $\mathbf{X}_{0}$ to first go through the transformation, $s(\mathbf{x}_{0}, \theta_{s})$. This maintains generality and allows NODEs to be used as a component of a larger model. For example, $\mathbf{X}_{0}$ could go through a ResNet before the NODE, and then through a softmax classifier at the end (which is accounted for in the term $\displaystyle  \frac{\partial L}{\partial \mathbf{x}(t_{n})^{T}}$). Introduce the new variable $\mathbf{F}$
\begin{equation}
\label{eqn: F_first_order}
    \mathbf{F} = \int_{t_{0}}^{t_{n}}\dot{\mathbf{x}}(t)dt
    = \int_{t_{0}}^{t_{n}}\left(\dot{\mathbf{x}} + A(t)
    (\dot{\mathbf{x}}-f^{(v)})
    \right)dt + 
    B(\mathbf{x}(t_{0}) - s)
\end{equation}

These are equivalent because $(\dot{\mathbf{x}} - f^{(v)})$ and $(\mathbf{x}(t_{0}) - s)$ are both zero. This means the matrices, $A(t)$ and $B$, can be chosen freely (as long as they are well behaved, finite etc.), to make the computation easier. The gradients of $\mathbf{x}(t_{n})$ with respect to the parameters are
\begin{equation}
\label{eqn: dxdtheta}
    \frac{d\mathbf{x}(t_{n})}{d\theta_{f}} = \frac{d\mathbf{F}}{d\theta_{f}},
    \qquad
    \qquad
    \frac{d\mathbf{x}(t_{n})}{d\theta_{s}} = \frac{d\mathbf{F}}{d\theta_{s}} + \frac{ds(\mathbf{X}_{0}, \theta_{s})}{d\theta_{s}}
\end{equation}

Differentiating $\mathbf{F}$ with respect to a general parameter $\theta$
\begin{equation}
    \frac{d\mathbf{F}}{d\theta} = \int_{t_{0}}^{t_{n}}\frac{d\dot{\mathbf{x}}}{d\theta} dt
    + \int_{t_{0}}^{t{n}}A(t)\left(
    \frac{d\dot{\mathbf{x}}}{d\theta} - \frac{\partial f^{(v)}}{\partial \theta} - \frac{\partial f^{(v)}}{\partial \mathbf{x}^{T}}\frac{d\mathbf{x}}{d\theta}
    \right)dt
    +B\left(
    \frac{d\mathbf{x}(t_{0})}{d\theta} - \frac{ds}{d\theta}
    \right)
\end{equation}

Integrating by parts
\begin{equation}
    \int_{t_{0}}^{t_{n}}A(t)\frac{d\dot{\mathbf{x}}}{d\theta}dt =
    \left[
    A(t)\frac{d\mathbf{x}}{d\theta}
    \right]_{t_{0}}^{t_{n}} - \int_{t_{0}}^{t_{n}}\dot{A}(t)\frac{d\mathbf{x}}{d\theta}dt
\end{equation}

Substituting this in and using $\int_{t_{0}}^{t_{n}}\frac{d\dot{\mathbf{x}}}{d\theta}dt = [\frac{d\mathbf{x}}{d\theta}]^{t_{n}}_{t_{0}}$, gives
\begin{equation}
\label{eqn: dFdthet}
\begin{aligned}
    \frac{d\mathbf{F}}{d\theta} = 
    &\left(
    \frac{d\mathbf{x}}{d\theta} + A(t)\frac{d\mathbf{x}}{d\theta}
    \right)\Biggr\vert_{t_{n}}
    -\left(
    \frac{d\mathbf{x}}{d\theta} + A(t)\frac{d\mathbf{x}}{d\theta}
    \right)\Biggr\vert_{t_{0}}
    - \int_{t_{0}}^{t_{n}}A(t)\frac{\partial f^{(v)}}{\partial \theta}dt\\
    &- \int_{t_{0}}^{t_{n}}\left(
    \dot{A(t)}+A(t)\frac{\partial f^{(v)}}{\partial \mathbf{x}^{T}}
    \right)\frac{d\mathbf{x}}{d\theta}dt
    + B\left(\frac{d\mathbf{x}}{d\theta}\Biggr\vert_{t_{0}} - \frac{ds}{d\theta}
    \right)
\end{aligned}
\end{equation}

Using the freedom of choice of $A(t)$, let it follow the ODE
\begin{equation}
    \dot{A(t)} = -A(t)\frac{\partial f^{(v)}}{\partial \mathbf{x}^{T}},
    \qquad
    \qquad
    A(t_{n}) = -I
\end{equation}

Where $I$ is the identity matrix. Then the first term and second integral in Equation (\ref{eqn: dFdthet}) become zero, yielding
\begin{equation}
\label{eqn: dFdthet2}
    \frac{d\mathbf{F}}{d\theta} = \left(
    B-I - A(t_{0})
    \right)\frac{d\mathbf{x}}{d\theta}\Biggr\vert_{t_{0}}
    +\int_{t_{n}}^{t_{0}}A(t)\frac{\partial f^{(v)}}{\partial \theta}dt
    -B\frac{ds}{d\theta}
\end{equation}

Now using the freedom of choice of $B$, let it obey the equation
\begin{equation}
    B = I +A(t_{0})
\end{equation}

This makes the first term in Equation (\ref{eqn: dFdthet2}) zero. This gives the final form of $\displaystyle  \frac{d\mathbf{F}}{d\theta}$
\begin{equation}
    \frac{d\mathbf{F}}{d\theta} = \int_{t_{n}}^{t_{0}}A(t)\frac{\partial f^{(v)}}{\partial \theta}dt - ( I + A(t_{0}))
    \frac{ds}{d\theta}
\end{equation}

Subbing into Equation (\ref{eqn: dxdtheta}) and using the fact that $f^{(v)}$ has no $\theta_{s}$ dependence and $s$ has no $\theta_{f}$ dependence
\begin{equation}
    \frac{d\mathbf{x}(t_{n})}{d\theta_{f}} = \int_{t_{n}}^{t_{0}} A(t)\frac{\partial f^{(v)}(\mathbf{x}, t, \theta_{f})}{\partial \theta_{f}}dt
    ,\qquad\qquad
    \frac{d\mathbf{x}(t_{n})}{d\theta_{s}} = -A(t_{0})\frac{ds(\mathbf{X}_{0}, \theta_{s})}{d\theta_{s}}
\end{equation}

This leads to the gradients of the loss
\begin{equation}
\label{eqn: dLdtheta2}
    \frac{dL}{d\theta_{f}} = \frac{\partial L}{\partial \mathbf{x}(t_{n})^{T}}\int_{t_{n}}^{t_{0}}
    A(t)\frac{\partial f^{(v)}(\mathbf{x}, t, \theta_{f})}{\partial\theta_{f}}dt
    ,\qquad\qquad
    \frac{dL}{d\theta_{s}} = -\frac{\partial L}{\partial \mathbf{x}(t_{n})^{T}}A(t_{0})
    \frac{ds(\mathbf{X}_{0}, \theta_{s})}{d\theta_{s}}
\end{equation}

Subject to the ODE for $A(t)$
\begin{equation}
\label{eqn: ode_lambda_first}
    \dot{A}(t) = -A(t)\frac{\partial f^{(v)}(\mathbf{x}, t, \theta_{f})}{\partial \mathbf{x}},
    \qquad
    \qquad
    A(t_{n}) = -I
\end{equation}

Now introduce the adjoint state $\mathbf{r}(t)$
\begin{equation}
\label{eqn: introduce_adj}
    \mathbf{r}(t) = -A(t)^{T}\frac{\partial L}{\partial \mathbf{x}(t_{n})},
    \qquad
    \qquad
    \mathbf{r}(t)^{T} = -\frac{\partial L}{\partial \mathbf{x}(t_{n})^{T}}A(t)
\end{equation}

Using the fact that $\displaystyle  \frac{\partial L}{\partial \mathbf{x}(t_{n})}$ is constant with respect to time, the adjoint equations are obtained by applying the definition of the adjoint in Equation (\ref{eqn: introduce_adj}), to the gradients in Equation (\ref{eqn: dLdtheta2}), and multiplying the ODE in Equation (\ref{eqn: ode_lambda_first}) by the constant $-\frac{\partial L}{\partial \mathbf{x}(t_{n})}$
\begin{equation}
\label{eqn: first_order_adjoint_method_grads}
    \frac{dL}{d\theta_{f}} = -\int_{t_{n}}^{t_{0}}\mathbf{r}(t)^{T}\frac{\partial f^{(v)}(\mathbf{x}, t, \theta_{f})}{\partial \theta_{f}}dt,
    \qquad
    \qquad
    \frac{dL}{d\theta_{s}}= \mathbf{r}(t_{0})^{T}\frac{ds(\mathbf{X}_{0}, \theta_{s})}{d\theta_{s}} \\
\end{equation}
Where the adjoint $\mathbf{a}(t)$ follows the ODE
\begin{equation}
\label{eqn: first_order_adjoint_method_ode}
    \dot{\mathbf{r}}(t) = -\mathbf{r}(t)^{T}\frac{\partial f^{(v)}(\mathbf{x}, t, \theta_{f})}{\partial \mathbf{x}},
    \qquad
    \qquad
    \mathbf{r}(t_{n}) = \frac{\partial L}{\partial \mathbf{x}(t_{n})}
\end{equation}

The gradients are found by integrating the adjoint state, $\mathbf{r}$, and the real state, $\mathbf{x}$, backwards in time, which requires no intermediate values to be stored, using constant memory, a major benefit over traditional backpropagation.

These are the same equations that were derived by \citet{chen2018neural}, however this includes the addition of letting $\mathbf{x}(t_{0}) = s(\mathbf{X}_{0}, \theta_{s})$ giving the corresponding gradient, $\displaystyle  \frac{dL}{d\theta_{s}}$. Additionally, the derivation used by \citet{chen2018neural} is simpler but does not present an obvious way to extend the adjoint method to second order ODEs, which this derivation method can do, as shown next.

\subsection{Second Order Adjoint}

Using the same derivation method, but with a second order differential equation, a second order adjoint method is derived, according to the proposition from the main text:

\textbf{Proposition \ref{prop: second_order_adjoint}.}\textit{
The adjoint state $\mathbf{r}(t)$ of SONODEs follows the second order ODE
\begin{equation}
\begin{aligned}
    \ddot{\mathbf{r}} &= \mathbf{r}^{T}\frac{\partial f^{(a)}}{\partial \mathbf{x}}
    -\dot{\mathbf{r}}^{T}\frac{\partial f^{(a)}}{\partial \dot{\mathbf{x}}}
    -\mathbf{r}^{T}\frac{d}{dt}\Biggr(\frac{\partial f^{(a)}}{\partial \dot{\mathbf{x}}}\Biggr)
\end{aligned}
\end{equation}
and the gradients of the loss with respect to the parameters of the acceleration, $\theta_{f}$ are 
\begin{equation}
    \frac{dL}{d\theta_{f}} = -\int_{t_{n}}^{t_{0}}\mathbf{r}^{T}\frac{\partial f^{(a)}}{\partial \theta_{f}}dt,
\end{equation}
}

\begin{proof}
In general, the loss function, $L$, depends on $\mathbf{x}$ and $\dot{\mathbf{x}}$
\begin{equation}
\label{eqn: dLdthet_second_order_1}
    \frac{dL}{d\theta} = \frac{\partial L}{\partial \mathbf{x}(t_{n})^{T}}
    \frac{d\mathbf{x}(t_{n})}{d\theta}
    + \frac{\partial L}{\partial \dot{\mathbf{x}}(t_{n})^{T}}\frac{d\dot{\mathbf{x}}(t_{n})}{d\theta}
\end{equation}

The gradients from the positional part and the velocity part are found separately and added. Firstly the position
\begin{equation}
    \mathbf{x}(t_{n}) = \int_{t_{0}}^{t_{n}}\dot{\mathbf{x}}(t) dt + \mathbf{x}(t_{0})
\end{equation}
Subject to the second order ODE
\begin{equation}
    \ddot{\mathbf{x}} = f^{(a)}(\mathbf{x}, \dot{\mathbf{x}}, t, \theta_{f})
    ,\qquad
    \mathbf{x}(t_{0}) = s(\mathbf{X}_{0}, \theta_{s})
    ,\qquad
    \dot{\mathbf{x}}(t_{0}) = g(\mathbf{x}(t_{0}), \theta_{g})
\end{equation}

Following the same procedure as in first order, but including the initial condition for the velocity as well
\begin{equation}
\label{eqn: F_second_order}
    \mathbf{F} = \int_{t_{0}}^{t_{n}}\dot{\mathbf{x}} + A(t)(\ddot{\mathbf{x}} - f^{(a)})dt
    +B(\dot{\mathbf{x}}(t_{0}) - g) + C(\mathbf{x}(t_{0}) - s)
\end{equation}

As before, the vectors, $(\ddot{\mathbf{x}} - f^{(a)})$, $(\dot{\mathbf{x}}(t_{0}) - g)$ and $(\mathbf{x}(t_{0}) - s)$ are zero, which gives freedom to choose the matrices $A(t)$, $B$ and $C$ to make the calculation easier. The gradients of $\mathbf{x}(t_{n})$ with respect to the parameters $\theta$ are
\begin{equation}
\label{eqn: dxdtheta_second__order_1}
    \frac{d\mathbf{x}(t_{n})}{d\theta_{f}} = \frac{d\mathbf{F}}{d\theta_{f}},
    \qquad\qquad
    \frac{d\mathbf{x}(t_{n})}{d\theta_{g}} = \frac{d\mathbf{F}}{d\theta_{g}},
    \qquad\qquad
    \frac{d\mathbf{x}(t_{n})}{d\theta_{s}} = \frac{d\mathbf{F}}{d\theta_{s}} + \frac{ds(\mathbf{X}_{0}, \theta_{s})}{d\theta_{s}}
\end{equation}

Differentiating $F$ from equation \ref{eqn: F_second_order} with respect to a general parameter
\begin{equation}
\label{eqn: dFdtheta_second_order_1}
\begin{aligned}
    \frac{d\mathbf{F}}{d\theta} = & \left[ \frac{d\mathbf{x}}{d\theta} \right]_{t_{0}}^{t_{n}}  - \int_{t_{0}}^{t_{n}}A(t)\frac{\partial f^{(a)}}{\partial \theta}dt
    + \int_{t_{0}}^{t_{n}}A(t)\left(
    \frac{d\ddot{\mathbf{x}}}{d\theta}
    - \frac{\partial f^{(a)}}{\partial \mathbf{x}^{T}}\frac{d\mathbf{x}}{d\theta} - 
    \frac{\partial f^{(a)}}{\partial \dot{\mathbf{x}}^{T}}\frac{d\dot{\mathbf{x}}}{d\theta}
    \right)dt \\
    &+B\left(
    \frac{d\dot{\mathbf{x}}}{d\theta}\Biggr\vert_{t_{0}}- \frac{\partial g}{\partial \theta} - \frac{\partial g}{\partial \mathbf{x}(t_{0})^{T}}\frac{d\mathbf{x}(t_{0})}{d\theta}
    \right) 
    + C\left(
    \frac{d\mathbf{x}}{d\theta}\Biggr\vert_{t_{0}} - \frac{ds}{d\theta}
    \right)
\end{aligned}
\end{equation}

Integrating by parts
\begin{equation}
    \int_{t_{0}}^{t_{n}}A(t) \frac{d\ddot{\mathbf{x}}}{d\theta}dt = 
    \left[
    A(t)\frac{d\dot{\mathbf{x}}}{d\theta}-\dot{A}(t)
    \frac{d\mathbf{x}}{d\theta}
    \right]_{t_{0}}^{t_{n}}
    +\int_{t_{0}}^{t_{n}} \ddot{A}(t)
    \frac{d\mathbf{x}}{d\theta}dt    
\end{equation}

\begin{equation}
    \int_{t_{0}}^{t_{n}}A(t)\frac{\partial f^{(a)}}{\partial \dot{\mathbf{x}}^{T}}\frac{d\dot{\mathbf{x}}}{d\theta}dt =
    \left[
    A(t)\frac{\partial f^{(a)}}{\partial \dot{\mathbf{x}}^{T}}
    \frac{d\mathbf{x}}{d\theta}
    \right]_{t_{0}}^{t_{n}}
    - \int_{t_{0}}^{t_{n}}\frac{d}{dt}\left(
    A(t)\frac{\partial f^{(a)}}{\partial \dot{\mathbf{x}}^{T}}
    \right)\frac{d\mathbf{x}}{d\theta}dt
\end{equation}

Subbing these into Equation (\ref{eqn: dFdtheta_second_order_1})
\begin{equation}
\label{eqn: dF_dtheta_second_order_2}
\begin{aligned}
    \frac{d\mathbf{F}}{d\theta} = 
    &\left[
    \left( 
    I 
    -\dot{A}
    - A\frac{\partial f^{(a)}}{\partial \dot{\mathbf{x}}^{T}}
    \right)\frac{d\mathbf{x}}{d\theta}
    +A\frac{d\dot{\mathbf{x}}}{d\theta}
    \right]_{t_{n}}
    -\left[
    \left( 
    I 
    -\dot{A}
    - A\frac{\partial f^{(a)}}{\partial \dot{\mathbf{x}}}
    \right)\frac{d\mathbf{x}}{d\theta}
    +A\frac{d\dot{\mathbf{x}}}{d\theta}
    \right]_{t_{0}}\\
    & +\int_{t_{0}}^{t_{n}}\left(
    \ddot{A}(t) - A(t)\frac{\partial f^{(a)}}{\partial \mathbf{x}^{T}}
    +\frac{d}{dt}\left(
    A(t)\frac{\partial f^{(a)}}{\partial \dot{\mathbf{x}}^{T}}
    \right)
    \right)\frac{d\mathbf{x}}{d\theta}dt
    +\int_{t_{n}}^{t_{0}} A(t)\frac{\partial f^{(a)}}{\partial \theta}dt \\
    &+B\left(
    \frac{d\dot{\mathbf{x}}}{d\theta}\Biggr\vert_{t_{0}}- \frac{\partial g}{\partial \theta} - \frac{\partial g}{\partial \mathbf{x}(t_{0})^{T}}\frac{d\mathbf{x}(t_{0})}{d\theta}
    \right)
    + C\left(
    \frac{d\mathbf{x}}{d\theta}\Biggr\vert_{t_{0}} - \frac{ds}{d\theta}
    \right)
\end{aligned}
\end{equation}

Using the freedom to choose $A(t)$, let it follow the second order ODE
\begin{equation}
\label{eqn: lambda_ode_second_order}
    \ddot{A}(t) = A(t)\frac{\partial f^{(a)}}{\partial \mathbf{x}^{T}} - \frac{d}{dt}\left(
    A(t)\frac{\partial f^{(a)}}{\partial \dot{\mathbf{x}}^{T}}
    \right)
    ,\qquad
    A(t_{n}) = 0
    ,\qquad
    \dot{A}(t_{n}) = I
\end{equation}

This makes the first term and first integral in Equation (\ref{eqn: dF_dtheta_second_order_2}) zero, yielding 
\begin{equation}
\label{eqn: dFdtheta_second_order_3}
\begin{aligned}
    \frac{d\mathbf{F}}{d\theta} = 
    &\int_{t_{n}}^{t_{0}}A(t)\frac{\partial f^{(a)}}{\partial \theta}dt
    +\Biggr(\Biggr(
    \dot{A}(t) + A(t)\frac{\partial f^{(a)}}{\partial \dot{\mathbf{x}}^{T}} - I
    - B
    \frac{\partial g}{\partial \mathbf{x}(t_{0})^{T}}
    + C
    \Biggr)\frac{d\mathbf{x}}{d\theta}
    \Biggr)\Biggr\vert_{t_{0}} \\
    & +\left(\left(
    B - A\right)\frac{d\dot{\mathbf{x}}}{d\theta}
    \right)\Biggr\vert_{t_{0}}
    - B\frac{\partial g}{\partial \theta}
    - C\frac{d s}{d \theta}
\end{aligned}
\end{equation}

Now using the freedom of choice in $B$ and $C$
\begin{equation}
    B = A(t_{0}),
    \qquad\qquad
    C = -\dot{A}(t_{0}) -A(t_{0})\frac{\partial f^{(a)}}{\partial \dot{\mathbf{x}}}
    \Biggr\vert_{t_{0}}
    +I
    + A(t_{0})\frac{\partial g}{\partial \mathbf{x}(t_{0})^{T}}
\end{equation}

This makes the second and third terms in Equation (\ref{eqn: dFdtheta_second_order_3}) zero, yielding
\begin{equation}
    \frac{d\mathbf{F}}{d\theta} = 
    \int_{t_{n}}^{t_{0}}A(t)\frac{\partial f^{(a)}}{\partial \theta}dt
    -B\frac{\partial g}{\partial \theta}
    - C\frac{d s}{d \theta}
\end{equation}

These give the final gradients of $\mathbf{x}(t_{n})$ with respect to the parameters, by subbing the results for $B$, $C$ and $\displaystyle  \frac{d\mathbf{F}}{d\theta}$ above into Equation (\ref{eqn: dxdtheta_second__order_1}), using the fact that $f^{(a)}$, $g$ and $s$ only depend on the parameters $\theta_{f}$, $\theta_{g}$ and $\theta_{s}$ respectively
\begin{equation}
\begin{aligned}
    \frac{d\mathbf{x}(t_{n})}{d\theta_{f}} &= \int_{t_{n}}^{t_{0}}A(t)
    \frac{\partial f^{(a)}}{\partial \theta_{f}}dt
    ,\qquad\qquad
    \frac{d\mathbf{x}(t_{n})}{d\theta_{g}} = -A(t_{0})\frac{\partial g}{\partial \theta_{g}}
    \\
    \frac{d\mathbf{x}(t_{n})}{d\theta_{s}} &=  \left(
    \dot{A}(t_{0}) 
    + A(t_{0})
    \left(
    \frac{\partial f^{(a)}}{\partial \dot{\mathbf{x}}^{T}}\Biggr\vert_{t_{0}}
    - \frac{\partial g}{\partial \mathbf{x}(t_{0})^{T}}
    \right)
    \right)\frac{ds}{d\theta_{s}}
\end{aligned}
\end{equation}

As before, introduce the adjoint state $\mathbf{r}^{x}(t)$:
\begin{equation}
    \mathbf{r}^{x}(t) = -A(t)^{T}\frac{\partial L}{\partial \mathbf{x}(t_{n})},
    \qquad\qquad
    \mathbf{r}^{x}(t)^{T} = -\frac{\partial L}{\partial \mathbf{x}(t_{n})^{T}}A(t)
\end{equation}

Using the fact that $\displaystyle  \frac{\partial L}{\partial \mathbf{x}(t_{n})}$ is constant with respect to time, all the results above, and the ODE and initial conditions for $A(t)$ in Equation (\ref{eqn: lambda_ode_second_order}) can be multiplied by $\displaystyle -\frac{\partial L}{\partial \mathbf{x}(t_{n})^{T}}$, to get the gradients $\displaystyle  \frac{dL}{d\theta}$ in terms of $\mathbf{r}^{x}(t)$
\begin{equation}
\label{eqn: gradients_second_order_positional}
\begin{aligned}
    \frac{dL}{d\theta_{f}} &= -\int_{t_{n}}^{t_{0}}\mathbf{r}^{x}(t)^{T}
    \frac{\partial f^{(a)}}{\partial \theta_{f}}dt
    ,\qquad\qquad
    \frac{dL}{d\theta_{g}} = \mathbf{r}^{x}(t_{0})^{T}\frac{\partial g}{\partial \theta_{g}}
    \\
    \frac{dL}{d\theta_{s}} &=  \left(
    -\dot{\mathbf{r}}^{x}(t_{0})^{T} 
    - \mathbf{r}^{x}(t_{0})^{T}
    \left(
    \frac{\partial f^{(a)}}{\partial \dot{\mathbf{x}}^{T}}\Biggr\vert_{t_{0}}
    - \frac{\partial g}{\partial \mathbf{x}(t_{0})^{T}}
    \right)
    \right)\frac{d\mathbf{x}(t_{0})}{d\theta_{s}}
\end{aligned}
\end{equation}

Subject to the second order ODE for $\mathbf{r}^{x}(t)$
\begin{equation}
\label{eqn: d2adjdt2_position}
    \ddot{\mathbf{r}}^{x}(t) = \mathbf{r}^{x}(t)^{T}\frac{\partial f^{(a)}}{\partial \mathbf{x}} - \frac{d}{dt}\left(
    \mathbf{r}^{x}(t)^{T}\frac{\partial f^{(a)}}{\partial \dot{\mathbf{x}}}
    \right)
    ,\qquad
    \mathbf{r}^{x}(t_{n}) = 0
    ,\qquad
    \dot{\mathbf{r}}^{x}(t_{n}) = -\frac{\partial L}{\partial \mathbf{x}(t_{n})}
\end{equation}

Where after differentiating with the product rule the ODE in Equation (\ref{eqn: d2adjdt2_position}) becomes
\begin{equation}
\label{eqn: d2adjdt2_position_product_rule}
    \ddot{\mathbf{r}}^{x}(t) = \mathbf{r}^{x}(t)^{T}\frac{\partial f^{(a)}}{\partial \mathbf{x}}
    -\dot{\mathbf{r}}^{x}(t)^{T}\frac{\partial f^{(a)}}{\partial \dot{\mathbf{x}}}
    - \mathbf{r}^{x}(t)^{T}
    \left(
    \frac{d}{dt}\frac{\partial f^{(a)}}{\partial \dot{\mathbf{x}}}
    \right)
\end{equation}

Where doing the full time derivative gives
\begin{equation}
\label{eqn: full_time_derivative_f(a)}
    \frac{d}{dt}\Biggr(\frac{\partial f^{(a)}}{\partial \dot{\mathbf{x}}}\Biggr)=
    [\dot{\mathbf{x}}^{T}, f^{(a)T}, 1]
    \begin{bmatrix}
    \partial_{\mathbf{x}}
    \\
    \partial_{\dot{\mathbf{x}}}
    \\
    \partial_{t}
    \end{bmatrix}
    \Biggr(\frac{\partial f^{(a)}}{\partial \dot{\mathbf{x}}}\Biggr)
\end{equation}

Where the fact that $\ddot{x} = f^{(a)}$ has been used. This is only when the loss depends on the position. The same method is used to look at the velocity part in Equation (\ref{eqn: dLdthet_second_order_1})
\begin{equation}
\frac{dL}{d\theta} = \frac{\partial L}{\partial \dot{\mathbf{x}}(t_{n})^{T}}
\frac{d\dot{\mathbf{x}}(t_{n})}{d\theta}
\end{equation}

Where
\begin{equation}
\dot{\mathbf{x}}(t_{n}) = \int_{t_{0}}^{t_{n}}\ddot{\mathbf{x}}(t)dt + \dot{\mathbf{x}}(t_{0})
\end{equation}

The general method is to take this expression and add zeros, in the form of $A(t)$, $B$ and $C$ multiplied by the ODE and initial conditions, $(\ddot{\mathbf{x}}-f^{(a)})$, $(\dot{\mathbf{x}}(t_{0})-g)$ and $(\mathbf{x}(t_{0})-s)$. Then differentiate with respect to a general parameter $\theta$ and integrate by parts to get any integrals containing $\displaystyle  \frac{d\dot{\mathbf{x}}}{d\theta}$ or $\displaystyle  \frac{d\ddot{\mathbf{x}}}{d\theta}$ in terms of $\displaystyle  \frac{d\mathbf{x}}{d\theta}$. Then choose the ODE for $A(t)$ to remove any $\displaystyle  \frac{d\mathbf{x}}{d\theta}$ terms in the integral, and the initial conditions of $A(t_{n})$ to remove the boundary terms at $t_{n}$. Then $B$ and $C$ are chosen to remove the boundary terms at $t_{0}$. After doing this the gradients of $\dot{\mathbf{x}}$ with respect to the parameters are
\begin{equation}
\begin{aligned}
&\frac{d\dot{\mathbf{x}}(t_{n})}{d\theta_{f}}  = \int_{t_{n}}^{t_{0}}A(t)\frac{\partial f^{(a)}}{\partial \theta_{f}}dt
,\qquad\qquad
\frac{d\dot{\mathbf{x}}(t_{n})}{d\theta_{g}}  = -A(t_{0})
\frac{\partial g}{\partial \theta_{g}}
\\
&\frac{d\dot{\mathbf{x}}(t_{n})}{d\theta_{s}}  = \Biggr(
\dot{A}(t_{0}) + A(t_{0})
\frac{\partial f^{(a)}}{\partial \dot{\mathbf{x}}^{T}}\Biggr\vert_{t_{0}} 
- A(t_{0})\frac{\partial g}{\partial \mathbf{x}(t_{0})^{T}}
\Biggr)\frac{ds}{d\theta_{s}}
\end{aligned}
\end{equation}

Subject to the second order ODE for $A(t)$
\begin{equation}
\ddot{A}(t) = A(t)\frac{\partial f^{(a)}}{\partial \mathbf{x}^{T}} - \frac{d}{dt}\left(
    A(t)\frac{\partial f^{(a)}}{\partial \dot{\mathbf{x}}^{T}}
    \right) 
,\qquad
A(t_{n}) = -I
,\qquad
\dot{A}(t_{n}) = \frac{\partial f^{(a)}}{\partial \dot{\mathbf{x}}^{T}}\Biggr\vert_{t_{n}}
\end{equation}

Now introduce the state $\mathbf{r}^{v}(t)$
\begin{equation}
\mathbf{r}^{v}(t) = -\frac{\partial L}{\partial \dot{\mathbf{x}}(t_{n})^{T}}A(t),
\qquad\qquad
\mathbf{r}^{v}(t) = -A(t)^{T}\frac{\partial L}{\partial \dot{\mathbf{x}}(t_{n})}
\end{equation}

Which allows the gradients of the loss with respect to the parameters to be written as
\begin{equation}
\label{eqn: gradients_second_order_velocity}
\begin{aligned}
&\frac{dL}{d\theta_{f}} = -\int_{t_{n}}^{t_{0}}\mathbf{r}^{v}(t)^{T}\frac{\partial f^{(a)}}{\partial \theta_{f}}dt
,\qquad\qquad
\frac{dL}{d\theta_{g}} = \mathbf{r}^{v}(t_{0})^{T}\frac{\partial g}{\partial \theta_{g}}
\\
&\frac{dL}{d\theta_{s}} = \left(
\mathbf{r}^{v}(t_{0})^{T}\frac{\partial g}{\partial \mathbf{x}(t_{0})^{T}}
- \dot{\mathbf{r}}^{v}(t_{0})^{T} - \mathbf{r}^{v}(t_{0})^{T}
\frac{\partial f^{(a)}}{\partial \dot{\mathbf{x}}^{T}}\Biggr\vert_{t_{0}}
\right)\frac{ds}{d\theta_{s}}
\end{aligned}
\end{equation}

Where $\mathbf{r}^{v}$ follows the second order ODE and initial conditions
\begin{equation}
\label{eqn: d2adjdt2_velocity}
\begin{aligned}
&\ddot{\mathbf{r}}^{v}(t) = \mathbf{r}^{v}(t)^{T}\frac{\partial f^{(a)}}{\partial \mathbf{x}}
    -\dot{\mathbf{r}}^{v}(t)^{T}\frac{\partial f^{(a)}}{\partial \dot{\mathbf{x}}}
    - \mathbf{r}^{v}(t)^{T}\frac{d}{dt}
    \left(
    \frac{\partial f^{(a)}}{\partial \dot{\mathbf{x}}}
    \right)
    \\
    &\mathbf{r}^{v}(t_{n}) = \frac{\partial L}{\partial \dot{\mathbf{x}}(t_{n})}
    ,\qquad\qquad
    \dot{\mathbf{r}}^{v}(t_{n}) = -\frac{\partial L}{\partial \dot{\mathbf{x}}(t_{n})^{T}}
\frac{\partial f^{(a)}}{\partial \dot{\mathbf{x}}}\Biggr\vert_{t_{n}}
\end{aligned}
\end{equation}

Now adding the gradients from the $\mathbf{x}$ dependence and the $\dot{\mathbf{x}}$ dependence together. It can be seen that the gradients are the same in Equations (\ref{eqn: gradients_second_order_positional}) and (\ref{eqn: gradients_second_order_velocity}), but just swapping $\mathbf{r}^{x}$ and $\mathbf{r}^{v}$. Additionally, it can be seen from the ODEs for $\mathbf{r}^{x}$ and $\mathbf{r}^{v}$ in Equations (\ref{eqn: d2adjdt2_position_product_rule}) and (\ref{eqn: d2adjdt2_velocity}), that they are governed by the same, linear, second order ODE, with different initial conditions. Therefore the gradients, $\displaystyle  \frac{dL}{d\theta}$, can be written in terms of a new adjoint state, $\mathbf{r} = \mathbf{r}^{x}+\mathbf{r}^{v}$
\begin{equation}
\label{eqn: gradients_from_second_order_method_final}
\begin{aligned}
&\frac{dL}{d\theta_{f}} = -\int_{t_{n}}^{t_{0}}\mathbf{r}(t)^{T}\frac{\partial f^{(a)}(\mathbf{x}, \dot{\mathbf{x}}, t, \theta_{f})}{\partial \theta_{f}}dt
,\qquad\qquad
\frac{dL}{d\theta_{g}} = \mathbf{r}(t_{0})^{T}\frac{\partial g(\mathbf{x}(t_{0}), \theta_{g})}{\partial \theta_{g}}
\\
&\frac{dL}{d\theta_{s}} = \left(
\mathbf{r}(t_{0})^{T}\frac{\partial g(\mathbf{x}(t_{0}), \theta_{g})}{\partial \mathbf{x}(t_{0})^{T}}
- \dot{\mathbf{r}}(t_{0})^{T} -\mathbf{r}(t_{0})^{T}
\frac{\partial f^{(a)}(\mathbf{x}, \dot{\mathbf{x}}, t, \theta_{f})}{\partial \dot{\mathbf{x}}^{T}}\Biggr\vert_{t_{0}}
\right)\frac{ds(\mathbf{X}_{0}, \theta_{s})}{d\theta_{s}}
\end{aligned}
\end{equation}

Where $\mathbf{a}$ follows the second order ODE with initial conditions
\begin{equation}
\label{eqn: adjoint_ode_initial_conditions_second_order_final}
\begin{aligned}
    &\ddot{\mathbf{r}}(t) = \mathbf{r}(t)^{T}\frac{\partial f^{(a)}(\mathbf{x}, \dot{\mathbf{x}}, t, \theta_{f})}{\partial \mathbf{x}}
    -\dot{\mathbf{r}}(t)\frac{\partial f^{(a)}(\mathbf{x}, \dot{\mathbf{x}}, t, \theta_{f})}{\partial \dot{\mathbf{x}}}
    - \mathbf{r}(t)^{T}\frac{d}{dt}\left(
    \frac{\partial f^{(a)}(\mathbf{x}, \dot{\mathbf{x}}, t, \theta_{f})}{\partial \dot{\mathbf{x}}}
    \right)
    \\
    &\mathbf{r}(t_{n}) = \frac{\partial L}{\partial \dot{\mathbf{x}}(t_{n})}
    ,\qquad\qquad
    \dot{\mathbf{r}}(t_{n}) = -\frac{\partial L}{\partial \mathbf{x}(t_{n})} -\frac{\partial L}{\partial \dot{\mathbf{x}}(t_{n})^{T}}
    \frac{\partial f^{(a)}(\mathbf{x}, \dot{\mathbf{x}}, t, \theta_{f})}{\partial \dot{\mathbf{x}}}\Biggr\vert_{t_{n}}
\end{aligned}
\end{equation}

The full derivative, $d_{t}(\partial_{\dot{\mathbf{x}}}f^{(a)})$, is given by Equation (\ref{eqn: full_time_derivative_f(a)}). The ODE can also be written compactly as
\begin{equation}
\label{eqn: compact_ode_second_order_adjoint}
    \ddot{\mathbf{r}}(t) = \mathbf{r}(t)^{T}\frac{\partial f^{(a)}(\mathbf{x}, \dot{\mathbf{x}}, t, \theta_{f})}{\partial \mathbf{x}} - \frac{d}{dt}\left(
    \mathbf{r}(t)^{T}\frac{\partial f^{(a)}(\mathbf{x}, \dot{\mathbf{x}}, t, \theta_{f})}{\partial \dot{\mathbf{x}}}
    \right) 
\end{equation}

Just as in the first order method, a sum over times stamps $t_{n}$ may be required. This matches and extends on the gradients and ODE given by proposition \ref{prop: second_order_adjoint}.
\end{proof}

\subsection{Equivalence between the two Adjoint methods}

When acting on a concatenated state, $\mathbf{z}(t) = [\mathbf{x}(t), \mathbf{v}(t)]$, the first order adjoint method will produce the same gradients as the second order adjoint method. However, it is more computationally efficient to use the first order method. This is also given in the main text as the following proposition:

\textbf{Proposition \ref{prop: adjoints_are_equivalent}.}\textit{
The gradient of $\theta_f$ computed through the adjoint of the coupled ODE from (\ref{eq:sonode_coupled}) and the gradient from (\ref{eq:sonode_grad}) are equivalent. However, the latter requires at least as many matrix multiplications as the former. }

Intuitively, the first order method will produce the same gradients because second order dynamics can be thought of as two coupled first order ODEs, where the first order dynamics happen in phase space. However, this provides no information about computational efficiency. We prove the equivalence and compare the computational efficiencies below.

\begin{proof}
The first order formulation of second order dynamics can be written as
\begin{equation}
    \mathbf{z}(t) = \begin{bmatrix}
           \mathbf{x}(t) \\
           \mathbf{v}(t) \\
         \end{bmatrix},
        \qquad
        \quad
    \dot{\mathbf{z}} = \begin{bmatrix}
        \mathbf{v} \\
        f^{(a)}(\mathbf{x}, \mathbf{v}, t, \theta_{f}) \\
     \end{bmatrix},
     \qquad
     \quad
     \mathbf{z}(t_{0}) = \begin{bmatrix}
           \mathbf{x}(t_{0}) \\
           \mathbf{v}(t_{0}) \\
         \end{bmatrix}
         = \begin{bmatrix}
           s(\mathbf{X}_{0}, \theta_{s}) \\
           g(s(\mathbf{X}_{0}, \theta_{s}), \theta_{g}) \\
         \end{bmatrix}
\end{equation}

When using index notation, $x_{i}$ and $v_{i}$ are concatenated to make $z_{i.}$. For $x_{i}$ and $v_{i}$, the index, i, ranges from 1 to $d$, whereas for $z_{i}$ it ranges from 1 to 2$d$ accounting for the concatenation. This is represented below
\begin{equation}
\label{eqn: z_as_x_v}
    z_{i}= 
\begin{dcases}
    x_{i},& \text{if } i\leq d\\
    v_{(i-d)},              & \text{if } i\geq d+1 \\
\end{dcases}
\end{equation}

It also extends to $\dot{z}_{i}$ and $z_{i}(t_{0})$, where $f^{(a)}_{i}$, $s_{i}$ and $g_{i}$ also have the index range from 1 to $d$, but the index of $\dot{z}_{i}$ goes from 1 to 2$d$ just like for $z_{i}$.
\begin{equation}
\label{eqn: z_dot_as_v_f}
    \dot{z}_{i}= \tilde{f}^{(v)}_{i}(\mathbf{z}, t, \tilde{\theta}_{f})=
\begin{dcases}
    v_{i},& \text{if } i\leq d\\
    f^{(a)}_{(i-d)}(\mathbf{x}, \mathbf{v}, t, \theta_{f}),              & \text{if } i\geq d+1 \\
\end{dcases}
\end{equation}

\begin{equation}
\label{eqn: z0_as_s_g}
    z_{i}(t_{0})= \tilde{s}_{i}(\mathbf{X}_{0}, \tilde{\theta}_{s})=
\begin{dcases}
    s_{i}(\mathbf{X}_{0}, \theta_{s}),& \text{if } i\leq d\\
    g_{(i-d)}(s(\mathbf{X}_{0}, \theta_{s}),\theta_{g}),              & \text{if } i\geq d+1 \\
\end{dcases}
\end{equation}

Using the first order adjoint method, Equations (\ref{eqn: first_order_adjoint_method_grads}) and (\ref{eqn: first_order_adjoint_method_ode}), and using index notation with repeated indices summed over, the gradients are
\begin{equation}
    \frac{dL}{d\tilde{\theta}_{f}} = -\int_{t{n}}^{t_{0}}r_{i}(t)\frac{\partial \tilde{f}^{(v)}_{i}(\mathbf{z}, t, \tilde{\theta}_{f})}{\partial \tilde{\theta}_{f}}dt,
    \qquad\qquad
    \frac{dL}{d\tilde{\theta}_{s}} = r_{i}(t_{0})
    \frac{d\tilde{s}_{i}(\mathbf{X}_{0}, \tilde{\theta}_{s})}{d\tilde{\theta}_{s}}
\end{equation}

Where the adjoint follows the ODE
\begin{equation}
    \dot{r}_{i}(t) = -r_{j}(t)\frac{\partial \tilde{f}^{(v)}_{j}(\mathbf{z}, t, \tilde{\theta}_{f})}{\partial z_{i}},
    \qquad\qquad
    r_{i}(t_{n}) = \frac{\partial L}{\partial z_{i}(t_{n})}
\end{equation}

Where just like in $z_{i}$, the index, i, ranges from 1 to 2$d$ in the adjoint $r_{i}(t)$. When writing the sum over the index explicitly
\begin{equation}
\label{eqn: explicit_sum_for_a_gradient}
\dot{r}_{i} = -\sum_{j = 1}^{2d}r_{j}\frac{\partial \tilde{f}^{(v)}_{j}}{\partial z_{i}}
\qquad
= -\sum_{j = 1}^{d}r_{j}\frac{\partial \tilde{f}^{(v)}_{j}}{\partial z_{i}}
-\sum_{j = d+1}^{2d}r_{j}\frac{\partial \tilde{f}^{(v)}_{j}}{\partial z_{i}}
\end{equation}

Now split up the adjoint state, $\mathbf{r}$, into two equally sized vectors, $\mathbf{r}^{A}$ and $\mathbf{r}^{B}$, where their indices only range from 1 to $d$, like $\mathbf{x}$, $\mathbf{v}$, $f^{(a)}$, $g$ and $s$.
\begin{equation}
\label{eqn: a_as_aa_ab}
    r_{i}= 
\begin{dcases}
    r^{A}_{i},         & \text{if } i\leq d\\
    r^{B}_{(i-d)},     & \text{if } i\geq d+1 \\
\end{dcases}
\end{equation}

Using Equations (\ref{eqn: z_as_x_v}), (\ref{eqn: z_dot_as_v_f}), (\ref{eqn: z0_as_s_g}) and (\ref{eqn: a_as_aa_ab}), and subbing them into Equation (\ref{eqn: explicit_sum_for_a_gradient}), the derivative can be written as
\begin{equation}
\dot{r}_{i} = -\sum_{j = 1}^{d}r^{A}_{j}\frac{\partial v_{j}}{\partial z_{i}}
-\sum_{j = d+1}^{2d}r^{B}_{(j-d)}\frac{\partial f^{(a)}_{(j-d)}}{\partial z_{i}}
\end{equation}

Relabelling the indices in the second sum $(j-d) \xrightarrow{}j$
\begin{equation}
\dot{r}_{i} = -\sum_{j = 1}^{d}r^{A}_{j}\frac{\partial v_{j}}{\partial z_{i}}
-\sum_{j = 1}^{d}r^{B}_{j}\frac{\partial f^{(a)}_{j}}{\partial z_{i}}
\end{equation}

Looking at specific values of i:

$i \leq d$
\begin{equation}
\dot{r}_{i} = \dot{r}^{A}_{i} = -\sum_{j = 1}^{d}r^{A}_{j}\frac{\partial v_{j}}{\partial x_{i}}
-\sum_{j = 1}^{d}r^{B}_{j}\frac{\partial f^{(a)}_{j}}{\partial x_{i}}
,\qquad
= -\sum_{j = 1}^{d}r^{B}_{j}\frac{\partial f^{(a)}_{j}}{\partial x_{i}}
\end{equation}

$i \geq d+1$
\begin{equation}
\dot{r}_{i}=
\dot{r}^{B}_{(i-d)} = -\sum_{j = 1}^{d}r^{A}_{j}\frac{\partial v_{j}}{\partial v_{(i-d)}}
-\sum_{j = 1}^{d}r^{B}_{j}\frac{\partial f^{(a)}_{j}}{\partial v_{(i-d)}}
\end{equation}

Relabelling the first index $(i-d) \xrightarrow{} i$
\begin{equation}
\dot{r}^{B}_{i} = -\sum_{j = 1}^{d}r^{A}_{j}\frac{\partial v_{j}}{\partial v_{i}}
-\sum_{j = 1}^{d}r^{B}_{j}\frac{\partial f^{(a)}_{j}}{\partial v_{i}}
\end{equation}

Noting that, 
$\displaystyle  \frac{\partial v_{j}}{\partial v_{i}} = \delta_{ij}$, the time derivatives can be written in vector matrix notation as
\begin{equation}
    \dot{\mathbf{r}}^{A}(t) = -\mathbf{r}^{B}(t)^{T}\frac{\partial f^{(a)}(\mathbf{x}, \mathbf{v}, t, \theta_{f})}{\partial \mathbf{x}}
\end{equation}

\begin{equation}
\label{eqn: ab_dot}
    \dot{\mathbf{r}}^{B}(t) = -\mathbf{r}^{A}(t) -\mathbf{r}^{B}(t)^{T}    \frac{\partial f^{(a)}(\mathbf{x}, \mathbf{v}, t, \theta_{f})}{\partial \mathbf{v}}
\end{equation}

Differentiating Equation (\ref{eqn: ab_dot})
\begin{equation}
\label{eqn: ab_dot_dot}
    \ddot{\mathbf{r}}^{B}(t) = \mathbf{r}^{B}(t)^{T}\frac{\partial f^{(a)}(\mathbf{x}, \mathbf{v}, t, \theta_{f})}{\partial \mathbf{x}}
    -\frac{d}{dt}\left(
    \mathbf{r}^{B}(t)^{T}
    \frac{\partial f^{(a)}(\mathbf{x}, \mathbf{v}, t, \theta_{f})}{\partial \mathbf{v}}
    \right)
\end{equation}

This matches the ODE for the second order method in Equation (\ref{eqn: compact_ode_second_order_adjoint}). Now applying the initial conditions, using index notation again
\begin{equation}
r_{i}(t_{n}) = \frac{\partial L}{\partial z_{i}(t_{n})}
\end{equation}

For $i \leq d$
\begin{equation}
r_{i} = r^{A}_{i}(t_{n}) = \frac{\partial L}{\partial x_{i}(t_{n})}
\end{equation}

For $i \geq d+1$
\begin{equation}
r_{i}(t_{n}) = r^{B}_{(i-d)}(t_{n}) =\frac{\partial L}{\partial v_{(i-d)}(t_{n})}
\quad
\xrightarrow{}
\quad
r^{B}_{i}(t_{n}) = \frac{\partial L}{\partial v_{i}(t_{n})}
\end{equation}

Applying these initial conditions in $\mathbf{r}^{A}$ and $\mathbf{r}^{B}$ to Equation (\ref{eqn: ab_dot})
\begin{equation}
\dot{r}^{B}_{i}(t_{n}) = -\frac{\partial L}{\partial x_{i}(t_{n})} 
-\frac{\partial L}{\partial v_{j}(t_{n})}\frac{\partial f^{(a)}_{j}}{\partial v_{i}}\Biggr\vert_{t_{n}}
\end{equation}

By looking at the ODE and initial conditions, it is clear $\mathbf{r}^{B}$ is equivalent to the second order adjoint, in Equation (\ref{eqn: adjoint_ode_initial_conditions_second_order_final}). Now looking at the gradients, and including an explicit sum over the index

\begin{equation}
\frac{dL}{d\tilde{\theta}_{f}} = -\int_{t_{n}}^{t_{0}}
\sum_{i=1}^{2d}
r_{i}\frac{\partial \tilde{f}^{(v)}_{i}}{\partial \tilde{\theta}_{f}}dt
\quad
\xrightarrow[]{}
\quad
= -\int_{t_{n}}^{t_{0}}\sum_{i=1}^{d}
r^{A}_{i}\frac{\partial v_{i}}{\partial \tilde{\theta}_{f}}dt
-\int_{t_{n}}^{t_{0}}\sum_{i=d+1}^{2d}
r^{B}_{(i-d)}\frac{\partial f^{(a)}_{(i-d)}}{\partial \tilde{\theta}_{f}}dt
\end{equation}

The first term is zero because $v$ has no explicit $\theta$ dependence. The second term, after relabelling and using summation convention becomes
\begin{equation}
\frac{dL}{d\theta_{f}} = -\int_{t_{n}}^{t_{0}}
r^{B}_{i}(t)\frac{\partial f^{(a)}_{i}}{\partial \theta_{f}}dt
\qquad
= -\int_{t_{n}}^{t_{0}}
\mathbf{r}^{B}(t)^{T}\frac{\partial f^{(a)}}{\partial \theta_{f}}dt
\end{equation}

Where $\tilde{\theta}_{f} = \theta_{f}$ has been used, as they are both the parameters for the acceleration. This matches the result for gradients of parameters in the acceleration term $\theta_{f}$, when using the second order adjoint method, because $\mathbf{r}^{B}$ is the adjoint.

Looking at the gradients related to the initial conditions
\begin{equation}
\frac{dL}{d\tilde{\theta}_{s}} = \mathbf{r}(t_{0})^{T}\frac{d\tilde{s}(\mathbf{X}_{0}, \tilde{\theta}_{s})}{d\tilde{\theta}_{s}}
\end{equation}

After going through the previous process of separating out the sums from $1\xrightarrow{}d$ and $d+1 \xrightarrow{}2d$, then relabelling the indices on $\mathbf{r}^{B}$, this becomes
\begin{equation}
= r^{A}_{i}(t_{0})\frac{ds_{i}(\mathbf{X}_{0}, \theta_{s})}{d\tilde{\theta}_{s}}
+ r^{B}_{i}(t_{0})\frac{dg_{i}(s(\mathbf{X}_{0}, \theta_{s}), \theta_{g})}{d\tilde{\theta}_{s}}
\end{equation}

Using the expression for $\mathbf{r}^{A}$ by rearranging Equation (\ref{eqn: ab_dot}), this can be written as
\begin{equation}
\label{eqn: tilde_theta_s_grads}
\frac{dL}{d\tilde{\theta}_{s}}= 
\left(
-\dot{r}^{B}_{i}(t_{0}) - r^{B}_{j}(t_{0})\frac{\partial f^{(a)}_{j}}{\partial v_{i}}\Biggr\vert_{t_{0}}
\right)
\frac{ds_{i}}{d\tilde{\theta}_{s}}
+ r^{B}_{i}(t_{0})\frac{dg_{i}}{d\tilde{\theta}_{s}}
\end{equation}

The parameters $\tilde{\theta}_{s}$ contain both $\theta_{s}$ and $\theta_{g}$. Looking at $\theta_{g}$ first, where $s(\mathbf{X}_{0}, \theta_{s})$ has no dependence
\begin{equation}
    \frac{dL}{d\theta_{g}} = r^{B}_{i}(t_{0})\frac{\partial g_{i}(s(\mathbf{X}_{0}, \theta_{s}),\theta_{g})}{\partial \theta_{g}}
    =
    \mathbf{r}^{B}(t_{0})^{T}\frac{\partial g(s(\mathbf{X}_{0},\theta_{s}),\theta_{g})}{\partial \theta_{g}}
\end{equation}

where $\displaystyle  \frac{dg}{d\theta_{g}}$ can be written as a partial derivative, because $\mathbf{X}_{0}$ and $\theta_{s}$ have no dependence on $\theta_{g}$ at all. This expression is equivalent to $\displaystyle  \frac{dL}{d\theta_{g}}$ found using the second order adjoint method. Now looking at the parameters $\theta_{s}$, these parameters are in $s(\mathbf{X}_{0}, \theta_{s})$ explicitly and $g(s, \theta_{g})$, implicitly through $s$. Subbing $\tilde{\theta}_{s} = \theta_{s}$ into Equation (\ref{eqn: tilde_theta_s_grads}) gives

\begin{equation}
\frac{dL}{d\theta_{s}}=
\left(
-\dot{r}^{B}_{i}(t_{0}) - r^{B}_{j}(t_{0})\frac{\partial f^{(a)}_{j}(\mathbf{x},\mathbf{v}, t, \theta_{f})}{\partial v_{i}}\Biggr\vert_{t_{0}}
+ r^{B}_{j}(t_{0})\frac{\partial g_{j}(s(\mathbf{X}_{0}, \theta_{s}), \theta_{g})}
{\partial s_{i}}
 \right)
\frac{ds_{i}(\mathbf{X}_{0}, \theta_{s})}{d\theta_{s}}
\end{equation}

Using the fact that $\mathbf{x}(t_0) = s$, this is the same result for $\displaystyle  \frac{dL}{d\theta_{s}}$ found using the second order adjoint method: 
\begin{equation}
    \frac{dL}{d\theta_{s}}=\left(
    \mathbf{r}^{B}(t_{0})^{T}
    \frac{\partial g(\mathbf{x}(t_{0}), \theta_{g})}{\partial \mathbf{x}(t_{0})^{T}}
    -\dot{\mathbf{r}}^{B}(t_{0})^{T} 
    - \mathbf{r}^{B}(t_{0})^{T}
    \frac{\partial f^{(a)}(\mathbf{x},\mathbf{v},t,\theta_{f})}{\partial \mathbf{v}^{T}}\Biggr\vert_{t_{0}} 
    \right)
    \frac{ds(\mathbf{X}_{0}, \theta_{s})}{d\theta_{s}}
\end{equation}

All of the gradients match, so the first order adjoint method acting on $\mathbf{z}(t) = [\mathbf{x}(t), \mathbf{v}(t)]$ will produce the same gradients as the second order adjoint method acting on $\mathbf{x}(t)$. Given by Equation (\ref{eqn: gradients_from_second_order_method_final}).

Looking at the efficiencies of each method and how they would be implemented. Both methods would integrate the state $\mathbf{z} = [\mathbf{x}, \mathbf{v}]$ forward in time, with $\dot{\mathbf{z}} = [\mathbf{v}, f^{(a)}]$. Both methods then integrate $\mathbf{z}$ and the adjoint backwards, in the same way. The difference is how the adjoint is represented. In first order it is represented as $[\mathbf{r}^{A}, \mathbf{r}^{B}]$ where $\mathbf{r}^{B}$ is the adjoint, in second order it is represented as $[\mathbf{r}, \dot{\mathbf{r}}]$ where $\mathbf{r}$ is the adjoint.

The time derivatives and initial conditions for the first order adjoint representation are
\begin{equation}
\label{eq: first_order_adjoint_final}
\begin{aligned}
    & \frac{d}{dt}\mathbf{r}^{A}(t) = -\mathbf{r}^{B}(t)^{T}\frac{\partial f^{(a)}(\mathbf{x}, \mathbf{v}, t, \theta_{f})}{\partial \mathbf{x}}
    \\
    & \frac{d}{dt}\mathbf{r}^{B}(t) = -\mathbf{r}^{A}(t) - \mathbf{r}^{B}(t)^{T}\frac{\partial f^{(a)}(\mathbf{x}, \mathbf{v}, t, \theta_{f})}{\partial \mathbf{v}}
    \\
    &\mathbf{r}^{A}(t_{n}) = \frac{\partial L}{\partial \mathbf{x}(t_{n})}
    \\
    &\mathbf{r}^{B}(t_{n}) = \frac{\partial L}{\partial \mathbf{v}(t_{n})}
\end{aligned}
\end{equation}

The time derivatives and intial conditions for the second order adjoint representation are
\begin{equation}
\label{eq: second_order_adjoint_final}
\begin{aligned}
    & \frac{d}{dt}\mathbf{r}(t) = \dot{\mathbf{r}}(t)
    \\
    & \frac{d}{dt}\dot{\mathbf{r}}(t) = 
    \mathbf{r}(t)^{T}\frac{\partial f^{(a)}(\mathbf{x},\mathbf{v},t,\theta_{f})}{\partial \mathbf{x}}
    -\dot{\mathbf{r}}(t)^{T}\frac{\partial f^{(a)}(\mathbf{x},\mathbf{v},t,\theta_{f})}{\partial \mathbf{v}}
    - \mathbf{r}(t)^{T}\frac{d}{dt}\left(
    \frac{\partial f^{(a)}(\mathbf{x},\mathbf{v},t,\theta_{f})}{\partial \mathbf{v}}
    \right)
    \\
    &\mathbf{r}(t_{n}) = \frac{\partial L}{\partial \mathbf{v}(t_{n})}
    \\
    &\dot{\mathbf{r}}(t_{n}) = -\frac{\partial L}{\partial \mathbf{x}(t_{n})}
    -\frac{\partial L}{\partial \mathbf{v}(t_{n})^{T}}
    \frac{\partial f^{(a)}(\mathbf{x},\mathbf{v}, t, \theta_{f})}{\partial \mathbf{v}}\Biggr\vert_{t_{n}}
\end{aligned}
\end{equation}

Where
\begin{equation}
    \frac{d}{dt}\Biggr(\frac{\partial f^{(a)}}{\partial \mathbf{v}}\Biggr)=
    [\mathbf{v}^{T}, f^{(a)T}, 1]
    \begin{bmatrix}
    \partial_{\mathbf{x}}
    \\
    \partial_{\mathbf{v}}
    \\
    \partial_{t}
    \end{bmatrix}
    \Biggr(\frac{\partial f^{(a)}}{\partial \mathbf{v}}\Biggr)
\end{equation}

Looking at Equations (\ref{eq: first_order_adjoint_final}) and (\ref{eq: second_order_adjoint_final}), the second order method has the additional term, $\mathbf{r}\cdot d_{t}(\partial_{\mathbf{v}}(f^{(a)}))$, in the ODE, and the additional term, $(\partial_{\mathbf{v}}L)\cdot (\partial_{\mathbf{v}}f^{(a)})$ in the initial conditions. The first order method acting on the concatenated state, $[\mathbf{x}, \mathbf{v}]$, requires equal or fewer matrix multiplications than the second order method acting on $\mathbf{x}$, to find the gradients at each step and the initial conditions. This is in the general case, but also for all specific cases, it is as efficient or more efficient. The same is also true for calculating the final gradients. 
\end{proof}

The reason for the difference in efficiencies is the state, $\mathbf{r}^{B}$, is the adjoint, and the state, $\mathbf{r}^{A}$, contains a lot of the complex information about the adjoint. It is an entangled representation of the adjoint, contrasting with the disentangled second order representation $[\mathbf{r}, \dot{\mathbf{r}}]$. This is similar to how ANODEs can learn an entangled representation of second order ODEs and SONODEs learn the disentangled representation, seen in Section \ref{sec: anodes_interpretability}. However, entangled representations are more useful here, because they do not need to be interpretable, they just need to produce the gradients, and the entangled representation can do this more efficiently.

This analysis provides useful information on the inner workings of the adjoint method. It shows a second order specific method does exist, but the first order method acting on a state $\mathbf{z}=[\mathbf{x},\mathbf{v}]$ will produce the same gradients more efficiently, due to how it represents the complexity. This was specific to second order ODEs, however, the first order adjoint will work on any system of ODEs, because any motion can be thought of as being first order motion in phase space. Additionally, the first order method may be the most efficient adjoint method. The complexity going from the first order to the second order was seen based on the calculation, so this is only likely to get worse as the system of ODEs becomes more complicated.

\section{Second Order ODEs are not Homeomorphisms}
\label{app: sonode_not_homeomorphism}

One of the conditions for a transformation to be a homeomorphism is for the transformation to be bijective (one-to-one and onto). In real space, a transformation that evolves according to a second order ODE does not have to be one-to-one. This is demonstrated using a one-dimensional counter-example
\[
\ddot{x} = 0
\qquad
\xrightarrow[]{}
\qquad
x(t) = x_{0} + v_{0}t
\]
\[
x_{0} = \begin{bmatrix}
        [0] \\
        [1] \\
    \end{bmatrix}
,\qquad\qquad
v_{0} = -x_{0} + 2 = \begin{bmatrix}
        [2]\\
        [1]\\
\end{bmatrix}
\]

If $t_{0} = 0$ and $t_{N}=1$
\[
x(1) = \begin{bmatrix}
        [2] \\
        [2] \\
    \end{bmatrix}
\]

So the transformation in real space is not always one-to-one, and therefore, not always a homeomorphism. 

\section{ANODEs learning 2nd Order}
\label{app: anode_learn_second_order}

Here we present the proofs for the propositions from Section \ref{sec: anodes_learn_2nd_order}

\subsection{Functional Form Proofs}

\textbf{Proposition \ref{prop: anode_general_form}.}\textit{The general form ANODEs learn second order behaviour is given by:
\begin{equation}
\begin{bmatrix}
\dot{\mathbf{x}}\\
\dot{\mathbf{a}}\\
\end{bmatrix}
=
\begin{bmatrix}
F(\mathbf{x},\mathbf{a},t,\theta_{F})\\
G(\mathbf{x},\mathbf{a},t,\theta_{G})\\
\end{bmatrix},
\qquad
G = \left( \frac{\partial F}{\partial \mathbf{a}^{T}}\right)_{\text{left}}^{-1}
\left(
f^{(a)} - \frac{\partial F}{\partial \mathbf{x}^{T}}F - 
\frac{\partial F}{\partial t}
\right)
\end{equation}}

\begin{proof}
Let $\mathbf{z}(t)$ be the state vector $[\mathbf{x}(t), \mathbf{a}(t)]$. The time derivatives can be written as
\begin{equation}
\begin{bmatrix}
\dot{\mathbf{x}}(t)\\
\dot{\mathbf{a}}(t)\\
\end{bmatrix}
=
\begin{bmatrix}
F(\mathbf{x},\mathbf{a}, t, \theta_{F})\\
G(\mathbf{x},\mathbf{a}, t, \theta_{G})\\
\end{bmatrix}
\end{equation}
Let $\mathbf{x}(t)$ follow the second order ODE, $\ddot{\mathbf{x}} = \dot{F} = f^{(a)}(\mathbf{x}, \dot{\mathbf{x}}, t, \theta_{f})$. Differentiating $F$ with respect to time
\begin{equation}
\dot{F} = \frac{\partial F}{\partial \mathbf{x}^{T}}\dot{\mathbf{x}} + \frac{\partial F}{\partial \mathbf{a}^{T}}\dot{\mathbf{a}} + \frac{\partial F}{\partial t}
=
f^{(a)}(\mathbf{x}, \dot{\mathbf{x}}, t, \theta_{f})
\end{equation}
Using $\dot{\mathbf{x}} = F$ and $\dot{\mathbf{a}} = G$
\begin{equation}
f^{(a)}(\mathbf{x}, F, t, \theta_{f})
=
\frac{\partial F}{\partial \mathbf{x}^{T}}F + \frac{\partial F}{\partial \mathbf{a}^{T}}G + \frac{\partial F}{\partial t}
\end{equation}

Rearranging for G
\begin{equation}
G(\mathbf{x},\mathbf{a}, t, \theta_{G}) = \left( \frac{\partial F}{\partial \mathbf{a}^{T}}\right)^{-1}_{\text{left}}
\left(
f^{(a)}(\mathbf{x}, F, t, \theta_{f}) - \frac{\partial F}{\partial \mathbf{x}^{T}}F - 
\frac{\partial F}{\partial t}
\right)
\end{equation}
\end{proof}

In order for the solution of $G$ to exist, the matrix $ \displaystyle \frac{\partial F}{\partial \mathbf{a}^{T}}$ must be invertible. Either the dimension of $\mathbf{a}$ matches $F$, $\mathbf{x}$ and $f^{(a)}$, so that $\displaystyle  \frac{\partial F}{\partial \mathbf{a}^{T}}$ is square, or $ \displaystyle \frac{\partial F}{\partial \mathbf{a}^{T}}$ has a left inverse. Crucially, $F$ must have explicit $\mathbf{a}$ dependence, or the inverse does not exist. Intuitively, in order for real space to couple to augmented space, there must be explicit dependence.

Using the equation for $G(\mathbf{x}, \mathbf{a}, t, \theta_{G})$, there is a gauge symmetry in the system, which proves proposition \ref{prop: anode_infinity}.

\textbf{Proposition \ref{prop: anode_infinity}.} \textit{ANODEs can learn an infinity of (non-trivial) functional forms to learn the true dynamics of a second order ODE in real space.}

\begin{proof}
Assume a solution for $F(\mathbf{x}, \mathbf{a}, t, \theta_{F})$ and $G(\mathbf{x},\mathbf{a}, t, \theta_{G})$ has been found such that, $\dot{F} = f^{(a)}$ and $F(\mathbf{x}_{0}, \mathbf{a}_{0}, t_{0}, \theta_{F}) = \dot{\mathbf{x}}_{0}$. If an arbitrary function of $\mathbf{x}$, $\phi(\mathbf{x})$, is added to $F$, where $\phi(\mathbf{x}_{0}) = 0$
\begin{equation}
    \tilde{F}(\mathbf{x}, \mathbf{a}, t, \theta_{F}) = F(\mathbf{x}, \mathbf{a}, t, \theta_{F}) + \phi(\mathbf{x})
\end{equation}
The initial velocity is still the same. The dynamics are preserved if there is a corresponding change in $G$
\begin{equation}
    \tilde{G}(\mathbf{x},\mathbf{a}, t, \theta_{G}) = \left( \frac{\partial (F+\phi)}{\partial \mathbf{a}^{T}}\right)^{-1}
\left(
f^{(a)}(\mathbf{x}, F+\phi, t, \theta_{f}) - \frac{\partial (F+\phi)}{\partial \mathbf{x}^{T}}(F+\phi) - 
\frac{\partial (F+\phi)}{\partial t}
\right)
\end{equation}
The proof can end here, however this can be simplified. $\phi(\mathbf{x})$ has no explicit $\mathbf{a}$ or $t$ dependence, so this equation simplifies to
\begin{equation}
    \tilde{G} = \left( \frac{\partial F}{\partial \mathbf{a}^{T}}\right)^{-1}
\left(
f^{(a)}(\mathbf{x}, F+\phi, t, \theta_{f}) - \frac{\partial F}{\partial \mathbf{x}^{T}}F - 
\frac{\partial F}{\partial t}
-
\frac{\partial F}{\partial \mathbf{x}^{T}}\phi -
\frac{\partial \phi}{\partial \mathbf{x}^{T}}F -
\frac{\partial \phi}{\partial \mathbf{x}^{T}}\phi
\right)
\end{equation}
The term $f^{(a)}(\mathbf{x}, F+\phi, t, \theta_{f})$ can be Taylor expanded (assuming convergence)
\begin{equation}
    f^{(a)}(\mathbf{x}, F+\phi, t, \theta_{f}) =
    f^{(a)}(\mathbf{x}, F, t, \theta_{f}) +
    \sum_{n=1}^{\infty}
    \left(
    \frac{\partial ^{n}f^{(a)}(\mathbf{x}, \dot{\mathbf{x}}, t, \theta_{f})}{\partial \dot{\mathbf{x}}^{Tn}}\Biggr\vert_{\dot{\mathbf{x}}=F}\frac{\phi^{n}}{n!}
    \right)
\end{equation}

Which gives the corresponding change in $G$
\begin{equation}
\label{eqn: g_gauge_change}
    \tilde{G} = G(\mathbf{x},\mathbf{a}, t, \theta_{G}) +
    \left(
    \frac{\partial F}{\partial \mathbf{a}^{T}}
    \right)^{-1}
    \left(
    \sum_{n=1}^{\infty}
    \left(
    \frac{\partial ^{n}f^{(a)}}{\partial \dot{\mathbf{x}}^{Tn}}\Biggr\vert_{\dot{\mathbf{x}}=F}\frac{\phi^{n}}{n!}
    \right)
    -\frac{\partial F}{\partial \mathbf{x}^{T}}\phi
    -\frac{\partial \phi}{\partial \mathbf{x}^{T}}F
    -\frac{\partial \phi}{\partial \mathbf{x}^{T}}\phi
    \right)
\end{equation}
\end{proof}

This demonstrates that there are infinite functional forms that ANODEs can learn. This only considered perturbing functions $\phi(\mathbf{x})$. More complex functions can be added that have $\mathbf{a}$ or $t$ dependence, which lead to a more complex change in $G$. By contrast, we now show SONODEs have a unique functional form.

\textbf{Proposition \ref{prop: sonode_unique}.} \textit{SONODEs learn to approximate a unique functional form to learn the true dynamics of a second order ODE in real space.}

\begin{proof}
Consider a dynamical system
\begin{equation}
    \frac{d^{2}\mathbf{x}}{dt^{2}} = f(\mathbf{x}, \mathbf{v}, t),
    \qquad\qquad
    \mathbf{x}(t_{0}) = \mathbf{x}_{0}, 
    \qquad\qquad
    \mathbf{v}(t_{0}) = \mathbf{v}_{0}
\end{equation}

For these problems we let the loss only depend on the position, if it depends on position and velocity there would be more restrictions. So if it is true when loss only depends on the position, it is also true when it depends on both position and velocity.

Assume that there is another system, that has the same position as a function of time
\begin{equation}
    \frac{d^{2}\tilde{\mathbf{x}}}{dt^{2}} = \tilde{f}(\tilde{\mathbf{x}}, \tilde{\mathbf{v}}, t),
    \qquad\qquad
    \tilde{\mathbf{x}}(t_{0}) = \tilde{\mathbf{x}}_{0}, 
    \qquad\qquad
    \tilde{\mathbf{x}}(t_{0}) = \tilde{\mathbf{v}}_{0}
\end{equation}

Where $f(\mathbf{x}, \mathbf{v}, t)\neq\tilde{f}(\tilde{\mathbf{x}}, \tilde{\mathbf{v}}, t)$. Because the initial conditions are given the position and velocity are defined at all times, and therefore position, velocity and acceleration can all be written as explicit functions of time. $\mathbf{x} \equiv \mathbf{x}(t)$, $\mathbf{v} \equiv \mathbf{v}(t)$. This allows for the acceleration to be written as a function of $t$ only, $f(\mathbf{x}, \mathbf{v}, t) = f_{\tau}(t)$ for all $t$. The same applies for the second system, $\tilde{\mathbf{x}} \equiv \tilde{\mathbf{x}}(t)$, $\tilde{\mathbf{v}} \equiv \tilde{\mathbf{v}}(t)$ and $\tilde{f}(\tilde{\mathbf{x}}, \tilde{\mathbf{v}}, t) = \tilde{f}_{\tau}(t)$

For all $t$, $\mathbf{x}(t) = \tilde{\mathbf{x}}(t)$, therefore, for any time increment, $\delta t$, $\mathbf{x}(t+\delta t) = \tilde{\mathbf{x}}(t+ \delta t)$. Taking the full time derivative of $\mathbf{x}$ and $\tilde{\mathbf{x}}(t)$
\begin{equation}
    \frac{d\mathbf{x}(t)}{dt} = \mathbf{v}(t) = \lim_{\delta t\to 0}
    \frac{\mathbf{x}(t+\delta t) - \mathbf{x}(t)}{\delta t}
\end{equation}
\begin{equation}
    \frac{d\tilde{\mathbf{x}}(t)}{dt} = \tilde{\mathbf{v}}(t) = \lim_{\delta t\to 0}
    \frac{\tilde{\mathbf{x}}(t+\delta t) - \tilde{\mathbf{x}}(t)}{\delta t}
\end{equation}

Using these two equations and the fact that $\mathbf{x}(t) = \tilde{\mathbf{x}}(t)$, it is inferred that $\mathbf{v}(t) = \tilde{\mathbf{v}}(t)$ for all $t$. Taking the full time derivative of $\mathbf{v}(t)$ and $\tilde{\mathbf{v}}(t)$
\begin{equation}
    \frac{d\mathbf{v}(t)}{dt} = f_{\tau}(t) = \lim_{\delta t\to 0}
    \frac{\mathbf{v}(t+\delta t) - \mathbf{v}(t)}{\delta t}
\end{equation}
\begin{equation}
    \frac{d\tilde{\mathbf{v}}(t)}{dt} = \tilde{f}_{\tau}(t) = \lim_{\delta t\to 0}
    \frac{\tilde{\mathbf{v}}(t+\delta t) - \tilde{\mathbf{v}}(t)}{\delta t}
\end{equation}

Using these two equation and the fact that $\mathbf{v}(t) = \tilde{\mathbf{v}}(t)$ for all $t$, it is also inferred that $f_{\tau}(t) = \tilde{f}_{\tau}(t)$ for all $t$.

Using these three facts, $\mathbf{x}(t) = \tilde{\mathbf{x}}(t)$, $\mathbf{v}(t) = \tilde{\mathbf{v}}(t)$ and $f_{\tau}(t) = \tilde{f}_{\tau}(t)$. It must also be true that $f(\mathbf{x}(t), \mathbf{v}(t), t) = \tilde{f}(\tilde{\mathbf{x}}(t), \tilde{\mathbf{v}}(t), t) \xrightarrow{} f(\mathbf{x}, \mathbf{v}, t) = \tilde{f}(\mathbf{x}, \mathbf{v}, t)$. Therefore the assumption that $f(\mathbf{x}, \mathbf{v}, t)\neq\tilde{f}(\tilde{\mathbf{x}}, \tilde{\mathbf{v}}, t)$ is incorrect, there can only be one functional form for $f(\mathbf{x}, \mathbf{v}, t)$.

Additionally, using $\mathbf{v}(t) = \tilde{\mathbf{v}}(t)$ for all $t$, the initial velocities must also be the same.

\end{proof}

\subsection{ANODEs Learning Two Functions}

In Section \ref{sec: how_anodes_learn_2nd}, it was shown that ANODEs were able to learn two functions at the same time
\begin{equation}
    x_{1}(t) = e^{-\gamma t}\sin(\omega t)
    ,\qquad\qquad
    x_{2}(t) = e^{-\gamma t}\cos(\omega t)
\end{equation}
using the solution
\begin{equation}
\begin{bmatrix}
\dot{x}\\
\dot{a}\\
\end{bmatrix}
=
\begin{bmatrix}
Ca -\omega x - \gamma x + \omega\\
\omega a - \gamma a  - \frac{1}{C}(2\omega^{2} x+ \gamma\omega -\omega^{2})\\
\end{bmatrix},
\end{equation}

This is a specific case of the general formulation given by Equation (\ref{eqn: anode_learn_2nd_order}). When the problem is generalised to have mixed amounts of sine and cosine in each function
\begin{equation}
x_{1}(t) = e^{-\gamma t}(A_{1}\sin(\omega t) + B_{1}\cos(\omega t)),
\qquad
\qquad
x_{2}(t) = e^{-\gamma t}(A_{2}\sin(\omega t) + B_{2}\cos(\omega t))
\end{equation}

ANODEs are still able to learn these functions, shown in the first plot of Figure~\ref{fig: anode_3_funcs}. As shown previously, if 
$F(\mathbf{x}, \mathbf{a}, t, \theta_{F})$ gets the addition, $\alpha x + \beta$, then the ODE is preserved if $G(\mathbf{x}, \mathbf{a}, t, \theta_{G})$ also gets the addition $\displaystyle  \frac{-1}{C}((\alpha - \omega + \gamma)(\alpha x + \beta)+ \alpha(C a - \omega x - \gamma x + \omega))$, given by Equation (\ref{eqn: g_gauge_change}). This gauge change preserves the ODE, but gives a new expression for the initial velocity
\begin{equation}
\dot{x}(0) = -\omega x(0) -\gamma x(0) + \omega + \alpha x(0) + \beta = \tilde{\alpha}x(0) + \tilde{\beta}
\end{equation}

which can be written in matrix-vector notation as
\begin{equation}
\label{eqn: matrix-vector}
    \begin{bmatrix}
    x_{1}(0) & 1 \\
    x_{2}(0) & 1 \\
    \end{bmatrix}
    \begin{bmatrix}
    \tilde{\alpha} \\
    \tilde{\beta} \\
    \end{bmatrix}
    =
    \begin{bmatrix}
    \dot{x}_{1}(0) \\
    \dot{x}_{2}(0) \\
    \end{bmatrix}
\end{equation}
There are two equations and two unknowns, $\tilde{\alpha}$ and $\tilde{\beta}$, so this is possible to solve, and for ANODEs to learn.\footnote{There are trivial cases where this would be impossible. For example if the two functions were $\pm sin(\omega t)$, they would have the same initial position, but different initial velocities. Corresponding to the matrix in Equation (\ref{eqn: matrix-vector}) having zero determinant.} To test this even further we added a third function to be learnt. ANODEs were able to do this, shown in the second plot of Figure~\ref{fig: anode_3_funcs}.\footnote{The figure also shows that when trajectories cross in real space they do not in augmented space, and when they cross in augmented space they do not in real space, supporting Proposition \ref{prop: no_crossing_different}.} 

\begin{figure}[h]
    \centering
    \includegraphics[width=\textwidth]{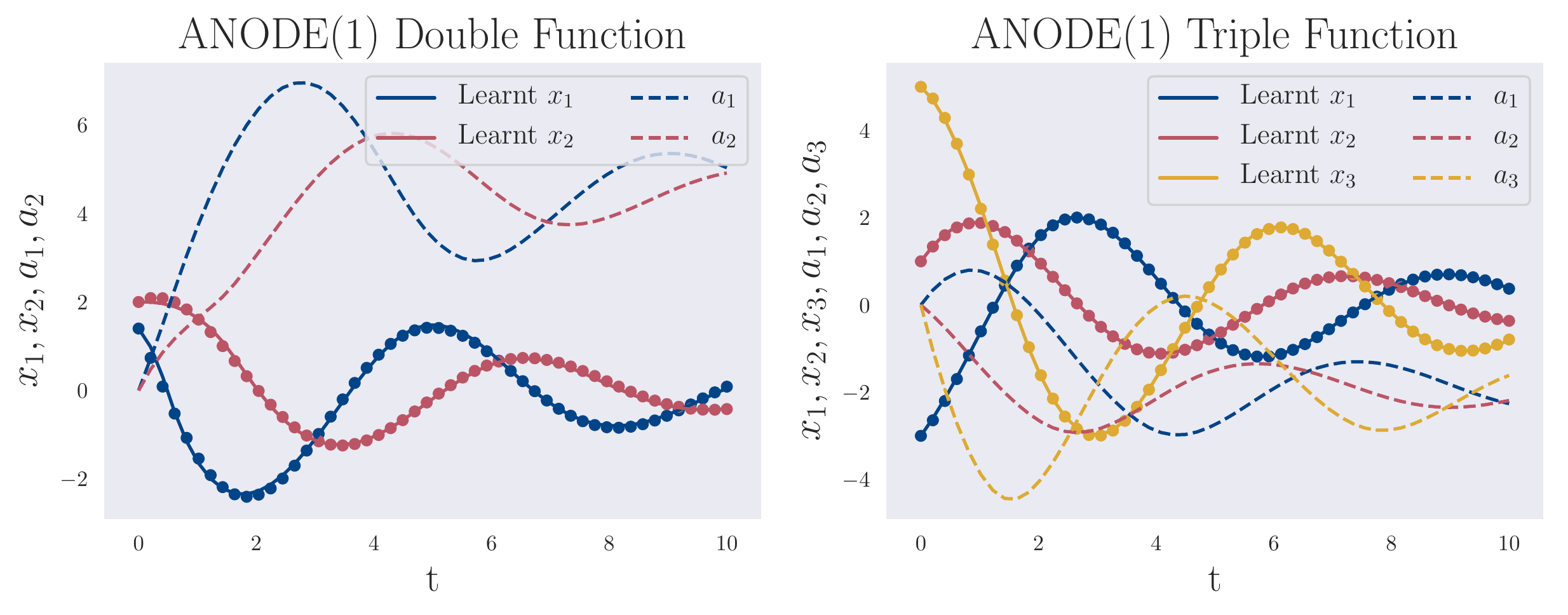}
    \caption{ANODE(1) learning two functions and three functions, with a shared ODE, but different initial conditions. The real trajectories are seen going through their sampled data points, and the corresponding augmented trajectories are also plotted. ANODE(1) is able to learn the trajectories.}
    \label{fig: anode_3_funcs}
\end{figure}

\section{Experimental Setup and Additional Results}
\label{app: experimental_setup}

We anticipate two main uses for SONODEs. One is using an experiment in a controlled environment, where the aim is to find values such as the coefficient of friction. The other use is when data is observed, and the aim is to extrapolate in time, but the experiment is not controlled, for example, observing weather. We would expect for the former, a simple model with only a single linear layer would be useful, to find those coefficients, and for the latter, a deeper model may be more appropriate. Additionally, Neural ODEs may be used in classification or other tasks that only involve the start and endpoints of the flow. For all of these tasks we used $t_{0}=0$ and $t_{1}=1$, and accelerations that were not time-dependent. For tasks depending on the start and endpoint only, a deeper neural network is more useful for the acceleration.

For all experiments, except the MNIST experiment, we optimise using Adam with a learning rate of 0.01. We also train on the complete datasets and do not minibatch. All the experiments were repeated 3 times to obtain a mean and standard deviation. Depending on the task at hand, we used two different architectures for NODEs, ANODEs and SONODEs. The first is a simple linear model, one weight matrix and bias without activations. This architecture, in the case of NODEs, ANODEs and SONODEs, was used on Silverbox, Airplane Vibrations and Van-Der-Pol Oscillator, with the aim of extracting coefficients from the models, for these tasks we also allowed ANODEs to learn the initial augmented position. The second architecture is a fully connected network with two hidden layers of size 20, it uses ELU activations in $\dot{\mathbf{z}}$ and tanh activations in the initial conditions. ELU and tanh were used because they allow for negative values in the ODE \cite{massaroli2020dissecting}.  

When considering ANODEs, they are in a higher-dimensional space than the problem, and the result must be projected down to the lower dimensions. This projection was not learnt as a linear layer, instead, the components were directly selected, using an identity for the real dimensions, and zero for the augmented dimensions. This was done because a final (or initial) learnt linear layer would hide the advantages of certain models. For example, the parity problem can be solved easily if NODEs are given a final linear layer, do not move the points and then multiply by -1. For this reason, no models used a linear layer at the end of the flow. Equally, they do not initialise with a linear layer as they again hide advantages. For example, the nested n-spheres problem, NODEs can solve this with an initial linear layer, if they were to go into a higher-dimensional space the points may already be linearly separated, as shown by \citet{massaroli2020dissecting}.

\subsection{Van Der Pol Oscillator}

ANODEs and SONODEs were tested on a forced Van Der Pol (VDP) Oscillator that exhibits a chaotic behaviour. More specfically, the parameters and equations of the particular VDP oscillator are:
\begin{equation}
    \ddot{x} = 8.53(1-x^{2})\dot{x} -x + 1.2\cos(0.2\pi t),
    \qquad\qquad
    x_{0} = 0.1,
    \qquad\qquad
    \dot{x}_{0} = 0
\end{equation}

\begin{figure}[ht]
    \centering
    \vspace{-7pt}
    \includegraphics[width=\textwidth]{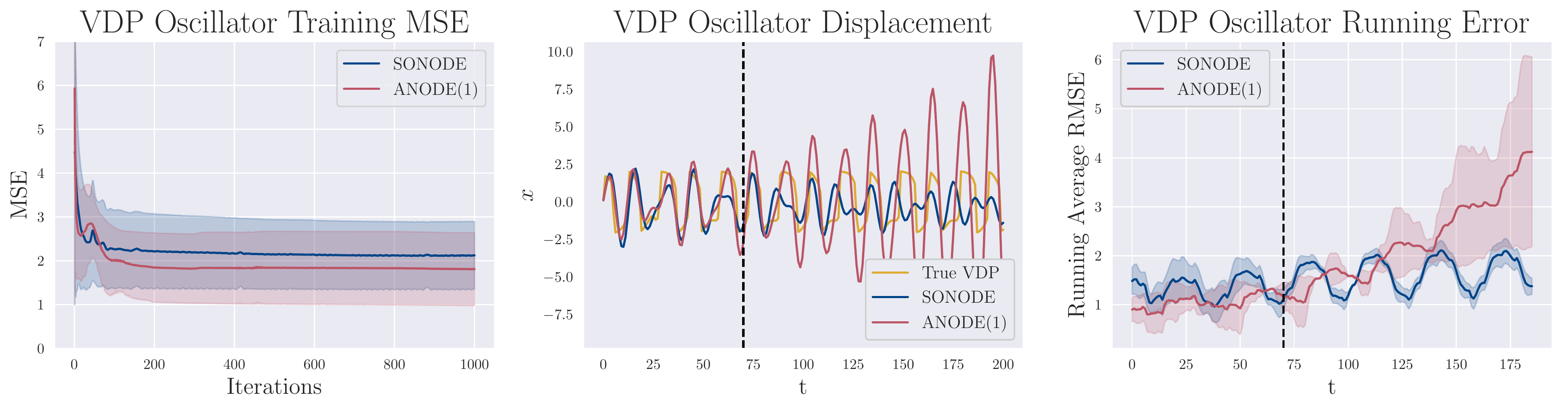}
    \vspace{-10pt}
    \caption{ANODE(1) and SONODE learning a Van-Der-Pol Oscillator. ANODEs are able to converge to a lower training loss, however they diverge when extrapolating. The models were trained on the first 70 points and extrapolated to 200.}
    \label{fig: vdp}
\end{figure}

As shown in Figure \ref{fig: vdp}, while ANODEs achieve a lower training loss than SONODEs, their test loss is much greater. We conjecture that, in the case of ANODEs, this is a case of overfitting. SONODEs, on the other hand, can better approximate the dynamics, therefore they exhibit better predictive performance. Note that, neither model can learn the VDP oscillator particularly well, which may be attributed to chaotic behaviour of the system at hand.

\subsection{Third Order NODEs on Airplane Vibrations}
\label{app: plane_part_2}
We test Third Order Neural ODEs (TONODEs) on the Airplane Vibrations task from section \ref{sec: planes}. The results are in Figure \ref{fig: plane_part_2}.

\begin{figure}[t]
    \centering
    \includegraphics[width=\textwidth]{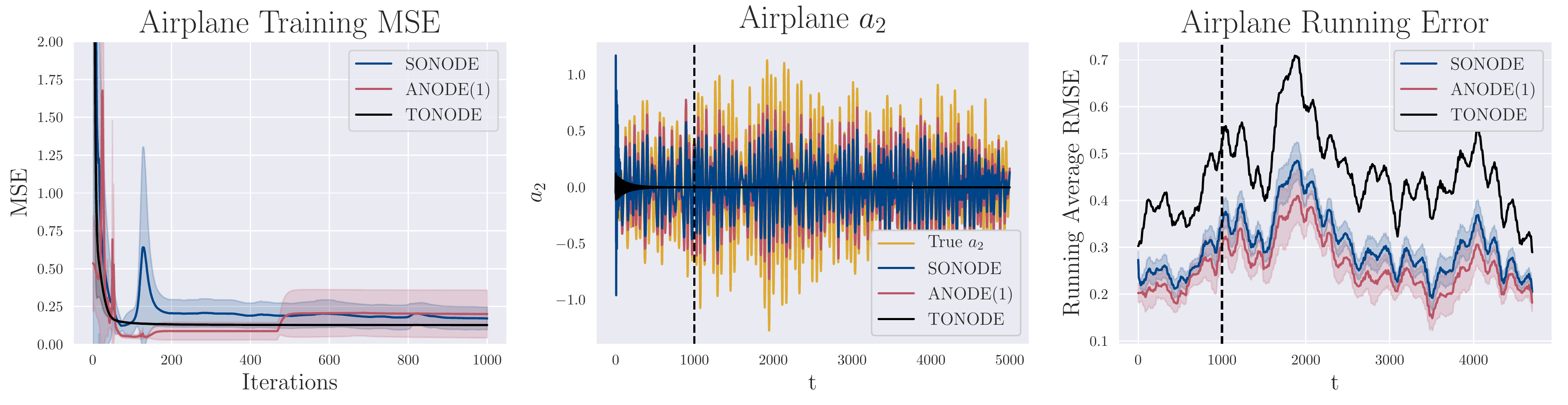}
    \caption{Repeating the Airplane Vibrations task with third order NODEs (TONODEs). We see that, in this case, TONODEs are not as successful at modelling these dynamics as SONODEs and ANODEs, having a larger error both on the training data and the extrapolation.}
    \label{fig: plane_part_2}
\end{figure}

We see that TONODEs vastly underperform compared to ANODEs and SONODEs. In each of the 3 repetitions of the experiment, the different initialisation found the best solution to be at zero. Therefore, whilst the loss stays constant, the error remains large. We hypothesise that despite theoretically being able to perform at least as well as SONODEs, TONODEs avoid exponentially growing at any point by exponentially decaying towards zero. It is likely that by rescaling the time to be between 0 and 1, TONODE would approach a more accurate solution.

\subsection{First Order Dynamics and Interpolation}

SONODEs contain a subset of models that is NODEs. Consider first order dynamics that is approximated by the NODE
\begin{equation}
\label{eq: node_interp}
    \dot{\mathbf{x}} = f^{(v)}(\mathbf{x}, t, \tilde{\theta}_{f})
\end{equation}
Carrying out the full time derivative of Equation (\ref{eq: node_interp}):
\begin{equation}
    \ddot{\mathbf{x}} = \frac{\partial f^{(v)}(\mathbf{x}, t, \tilde{\theta}_{f})}{\partial \mathbf{x}^{T}}\dot{\mathbf{x}}
    +
    \frac{\partial f^{(v)}(\mathbf{x}, t, \tilde{\theta}_{f})}{\partial t}
    ,\qquad
    \dot{\mathbf{x}}(t_{0}) = f^{(v)}(\mathbf{x}(t_{0}), t_{0}, \tilde{\theta}_{f})
\end{equation}

Which yields the SONODE equivalent of the learnt dynamics:
\begin{equation}
    f^{(a)}(\mathbf{x}, \mathbf{v}, t, \theta_{f}) = \frac{\partial f^{(v)}(\mathbf{x}, t, \tilde{\theta}_{f})}{\partial \mathbf{x}^{T}}\mathbf{v}
    +
    \frac{\partial f^{(v)}(\mathbf{x}, t, \tilde{\theta}_{f})}{\partial t}
    ,\qquad
    g(\mathbf{x}(t_{0}), \theta_{g}) = f^{(v)}(\mathbf{x}(t_{0}), t_{0}, \tilde{\theta}_{f})
\end{equation}

Additionally, it was shown in Equation (\ref{eq:sonode_coupled}) that SONODEs are a specific case of ANODEs that learn the initial augmented position. Therefore, anything that NODEs can learn, SONODEs should also be able to learn, and anything SONODEs can learn, ANODEs should be able to learn. To demonstrate that SONODEs and ANODEs can also learn first order dynamics, we task them with learning an exponential with no noise, $x(t) = exp(0.1667t)$. All models, as expected, are able to learn the function, as shown in Figure \ref{fig: interp}.

\begin{figure}[h]
    \centering
    \includegraphics[width=\textwidth]{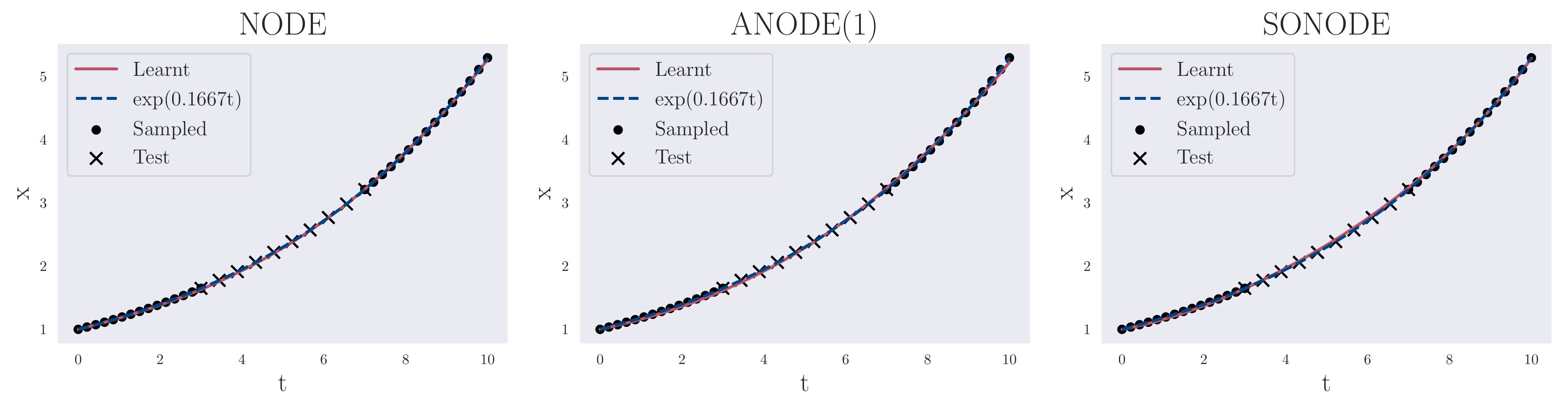}
    \caption{The different models learning an exponential, simple first order dynamics, and interpolating between two observation sections. As expected, all models are able to learn the function.}
    \label{fig: interp}
\end{figure}

\subsection{Performance on MNIST}

\label{app: mnist}
NODEs, SONODEs and ANODEs were tested on MNIST \cite{lecun1998gradient} to investigate their ability on classification tasks. The networks used convolutional layers, which in the case of SONODEs were used for both the acceleration and the initial velocity. ANODEs were augmented with one additional channel as is suggested by \citet{dupont2019augmented}. The models used a training batch size of 128 and test batch size of 1000, as well as group normalisation. SGD optimiser was used with a learning rate of 0.1 and momentum 0.9. The cross-entropy loss was used. The experiment was repeated 3 times with random initialisations to obtain a mean and standard deviation. The results are given in table \ref{tab: mnist} and Figure~\ref{fig: mnist}.

\begin{table}[h!]
    \centering
    \caption{Results for the MNIST experiments at convergence. SONODE converges to a higher test accuracy than NODEs with a lower NFE. ANODEs converge to the same higher test accuracy with a higher NFE, but with a lower parameter count than SONODEs.}
    \label{tab: mnist}
    \begin{tabular}{l cc}
    \toprule
        Model &  Test Accuracy & NFE\\
    \midrule
    NODE  & 0.9961 $\pm$ 0.0004 & 26.2 $\pm$ 0.0\\
    SONODE  &  \textbf{0.9963 $\pm$ 0.0001} & \textbf{20.1 $\pm$ 0.0}\\
    ANODE & \textbf{0.9963 $\pm$ 0.0001} & 32.2 $\pm$ 0.0 \\ 
    \bottomrule
    \end{tabular}
\end{table}

In terms of test accuracy, SONODEs and ANODEs perform marginally better than NODEs. ANODEs can achieve the same accuracy with fewer parameters than SONODEs because the dynamics are not limited to second order and it is only the final state that is of concern in classification. However, SONODEs are able to achieve the same accuracy with a lower number of function evaluations (NFE). NFE denotes how many function evaluations are made by the ODE solver, and represents the complexity of the learnt solution. It is a continuous analogue of the depth of a discrete layered network. In the case of NODEs and ANODEs, the NFE gradually increases meaning that the complexity of the flow also increases. However, in the case of SONODEs, the NFE stays constant, suggesting that the initial velocity was associated with larger gradients (otherwise we would expect NFE to increase for SONODEs with training).

\begin{figure}[h]
    \centering
    \includegraphics[width=\textwidth]{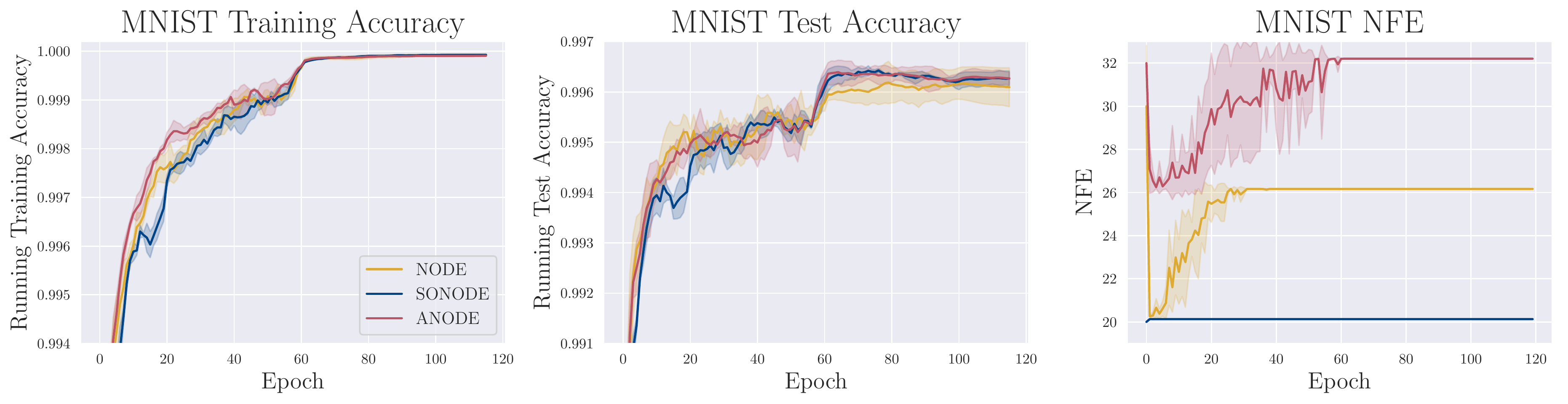}
    \caption{Comparing the performance of SONODEs and NODEs on the MNIST dataset. SONODEs converge to the same training accuracy and a higher test accuracy with a lower NFE than NODEs. NODEs had 208266 parameters, SONODEs had 283658 and ANODEs had 210626. Additional parameters were associated with the initial velocity, or the augmented channel.}
    \label{fig: mnist}
\end{figure}

\end{document}